\def\isanonymous{0} % TODO (camera-ready) set to 0
\def\doincludeappendixatbottom{1} % TODO (arxiv) set to 1

\def\year{2021}\relax

\documentclass[letterpaper]{article} % DO NOT CHANGE THIS
\usepackage{aaai21}  % DO NOT CHANGE THIS
\usepackage{times}  % DO NOT CHANGE THIS
\usepackage{helvet} % DO NOT CHANGE THIS
\usepackage{courier}  % DO NOT CHANGE THIS
\usepackage[hyphens]{url}  % DO NOT CHANGE THIS
\usepackage{graphicx} % DO NOT CHANGE THIS
\usepackage{natbib}  % DO NOT CHANGE THIS AND DO NOT ADD ANY OPTIONS TO IT
\usepackage{caption} % DO NOT CHANGE THIS AND DO NOT ADD ANY OPTIONS TO IT
\if\isanonymous1
\else
\fi
\usepackage[utf8]{inputenc} % allow utf-8 input
\usepackage{booktabs}       % professional-quality tables
\usepackage{nicefrac}       % compact symbols for 1/2, etc.
\usepackage{microtype}      % microtypography
\usepackage{amsmath,amsfonts,amssymb,amsthm}
\usepackage{mathtools}
\usepackage[vlined,ruled,algo2e]{algorithm2e}
\usepackage{algorithmic}
\usepackage{algorithm}
\usepackage{stackengine}
\usepackage[dvipsnames]{xcolor}
\usepackage{subcaption}
\usepackage{pdfpages}
\usepackage{enumitem}
\usepackage{xr}
\newcommand{\ie}{\emph{i.e.}}
\newcommand{\eg}{\emph{e.g.}}
\newcommand{\vs}{\emph{vs.}}

\newcommand{\wrt}{w.r.t.}
\newcommand{\st}{s.t.} % such that
\newcommand{\ub}{upper bound}

\newcommand{\UB}{Upper Bound}

\newcommand{\fp}{fixed-point}
\newcommand{\wc}{worst case}

\newcommand{\Qfun}{Q-function}
\newcommand{\Qval}{Q-value}

\renewcommand{\l}{Lifelong} % Michael's fix

\newcommand{\lrl}{\l{} RL}

\newtheorem{definition}{Definition}
\newtheorem{lemma}{Lemma}

\newtheorem{proposition}{Proposition}

\newtheorem{notation}{Notation}

\newtheorem*{definition*}{Definition}
\newtheorem*{lemma*}{Lemma}
\newtheorem*{property*}{Property}
\newtheorem*{proposition*}{Proposition}
\newtheorem*{theorem*}{Theorem}
\newtheorem*{corollary*}{Corollary}

\def\R{\mathbb{R}}  % real numbers
\def\N{\mathbb{N}}  % integers
\def\S{\mathcal{S}} % state space
\def\A{\mathcal{A}} % action space
\def\M{\mathcal{M}} % MDP space
\newcommand{\mdpvar}{\ensuremath{m}}
\newcommand{\mdpvarbar}{\ensuremath{\bar{m}}}
\newcommand{\setprobavect}[1]{\ensuremath{\mathcal{V}_{#1}}}

\newcommand{\V}[2]{\ensuremath{V^{#1}_{#2}}} % \pi M
\newcommand{\Q}[2]{\ensuremath{Q^{#1}_{#2}}} % \pi M
\newcommand{\rmax}{RMax}
\newcommand{\lrmax}{LRMax}
\newcommand{\maxqinit}{MaxQInit}
\newcommand{\lrmaxqinit}{LRMaxQInit}

\newcommand{\eqdef}{\ensuremath{\triangleq}}

\newcommand*{\qm}[1]{``#1''}

\newcommand{\intrange}[2]{\ensuremath{\left\{#1, \dots, #2\right\}}} % range of integers
\newcommand{\bigo}{\ensuremath{\mathcal{O}}} % big O
\newcommand{\bigotilde}{\ensuremath{\tilde{\mathcal{O}}}}
\newcommand{\wass}[1]{\ensuremath{W_{#1}}} % p
\newcommand{\wasserstein}[3]{\ensuremath{\wass{#1} \left( #2, #3 \right)}} % p mu nu
\newcommand{\functionspace}[2]{\mathcal{F} \left( #1, #2 \right)}
\newcommand{\FUNCTION}[5]{
	\ensuremath{
		\begin{array}{llll}
			#1: & #2 & \rightarrow & #3 \\
			& #4 & \mapsto & #5
		\end{array}
	}
}
\DeclarePairedDelimiter{\ceil}{\lceil}{\rceil}

\newcommand{\MAXENS}[2]{\ensuremath{ \max_{#1} \left( #2 \right) }}

\newcommand*{\SET}[1]{\ensuremath{ \left\{ #1 \right\} }}
\newcommand{\tuple}[2]{\ensuremath{\left( #1, #2 \right)}} % 2-uple
\renewcommand{\Pr}{\ensuremath{\textbf{Pr}}}

\renewcommand{\L}{\ensuremath{\mathcal{L}}}
\newcommand{\absnorm}[1]{\ensuremath{ \left| #1 \right| }}
\newcommand{\nnorm}[2]{\ensuremath{ \left\lVert #1 \right\rVert_{#2} }}
\newcommand{\onenorm}[1]{\nnorm{#1}{1}}

\newcommand{\inftynorm}[1]{\nnorm{#1}{\infty}}

\newcommand{\mthspc}{\ensuremath{\,}}
\newcommand{\phleq}{\hphantom{\leq\text{ }}}

\newcommand{\pheq}{\hphantom{=\text{ }}}

\renewcommand{\S}{\ensuremath{\mathcal{S}}}
\renewcommand{\A}{\ensuremath{\mathcal{A}}}
\newcommand{\SA}{\ensuremath{\S \times \A}}
\newcommand{\nS}{\ensuremath{S}} % number of states
\newcommand{\nA}{\ensuremath{A}} % number of actions

\newcommand{\transition}{\ensuremath{T}}
\newcommand{\tra}[3]{\ensuremath{\transition{}_{#1 #3}^{#2}}} % s a s'
\newcommand{\trahat}[3]{\ensuremath{\hat{\transition{}}_{#1 #3}^{#2}}}
\newcommand{\trabar}[3]{\ensuremath{\bar{\transition{}}_{#1 #3}^{#2}}}
\newcommand{\trahatbar}[3]{\ensuremath{\hat{\bar{\transition{}}}_{#1 #3}^{#2}}}
\newcommand{\Reward}{\ensuremath{R}}
\newcommand{\Rew}[2]{\ensuremath{\Reward{}_{#1}^{#2}}} % s a
\newcommand{\Rewhat}[2]{\ensuremath{\hat{\Reward{}}_{#1}^{#2}}}
\newcommand{\Rewbar}[2]{\ensuremath{\bar{\Reward{}}_{#1}^{#2}}}
\newcommand{\Rewhatbar}[2]{\ensuremath{\hat{\bar{\Reward{}}}_{#1}^{#2}}}
\newcommand{\Vmax}{\ensuremath{V_{\text{max}}}}

\newcommand{\prior}{\ensuremath{D_{\max}}} % prior knowledge

\newcommand{\ntimesteps}{\ensuremath{\tau}}
\newcommand{\pr}{\ensuremath{\textbf{Pr}}}
\newcommand{\p}{\ensuremath{p_{\min}}}
\newcommand{\pmin}{\ensuremath{p_{\min}}}

\newcommand{\modpm}[5]{\ensuremath{D_{#1 #2}^{#3} ( #4, #5 )}} % {s}{a}{f}{M1}{M2}
\newcommand{\modiv}[5]{\ensuremath{D_{#1 #2} ( #4 \| #5 ) }} % {s}{a}{f}{M1}{M2} -> light notation for \modpm when f = gamma times the optimal value function of {M2}
\newcommand{\modivhat}[4]{\ensuremath{\hat{D}_{#1 #2} ( #3 \| #4 )}} % {s}{a}{M1}{M2}

\newcommand{\mdpdiv}[4]{\ensuremath{d_{#1 #2} ( #3 \| #4 ) }} % {s}{a}{M1}{M2}
\newcommand{\mdpdivhat}[4]{\ensuremath{\hat{d}_{#1 #2} ( #3 \| #4 ) }} % {s}{a}{M1}{M2}
\newcommand{\mdppm}[4]{\ensuremath{\Delta_{#1 #2} ( #3, #4 ) }} % {s}{a}{M1}{M2}

\newcommand{\mdpdivpi}[5]{\ensuremath{d_{#1 #2}^{#3} ( #4 \| #5 ) }} % {s}{a}{pi}{M1}{M2}
\newcommand{\mdppmpi}[5]{\ensuremath{\Delta_{#1 #2}^{#3} ( #4, #5 ) }} % {s}{a}{pi}{M1}{M2}

\newcommand{\mdpdivn}[5]{\ensuremath{d_{#1 #2}^{#3} ( #4 \| #5 ) }} % {s}{a}{n}{M1}{M2}
\newcommand{\mdpdivhatn}[5]{\ensuremath{\hat{d}_{#1 #2}^{#3} ( #4 \| #5 ) }} % {s}{a}{n}{M1}{M2}

\title{Lipschitz Lifelong Reinforcement Learning}
\author{%
    \if\isanonymous1%
    Anonymized%
    \else%
    Erwan Lecarpentier\textsuperscript{\rm 1, 2}, David Abel\textsuperscript{\rm 3}, Kavosh Asadi\textsuperscript{\rm 3, 4\footnote{Kavosh Asadi finished working on this project before joining Amazon.}}, Yuu Jinnai\textsuperscript{\rm 3},\\Emmanuel Rachelson\textsuperscript{\rm 1}, Michael L. Littman\textsuperscript{\rm 3}\\%
    \fi%
}%
\affiliations{%
    \textsuperscript{\rm 1}ISAE-SUPAERO, Universit\'e de Toulouse, France\\%
    \textsuperscript{\rm 2}ONERA, The French Aerospace Lab, Toulouse, France\\%
    \textsuperscript{\rm 3}Brown University, Providence, Rhode Island, USA\\%
    \textsuperscript{\rm 4}Amazon Web Service, Palo Alto, California, USA\\%
    erwanlecarpentier@mailbox.org
}%

\begin{document}
	
%\linenumbers % TODO (camera-ready) remove line numbers

\maketitle

\begin{abstract}	
	We consider the problem of knowledge transfer when an agent is facing a series of Reinforcement Learning (RL) tasks.
	We introduce a novel metric between Markov Decision Processes and establish that close MDPs have close optimal value functions.
	Formally, the optimal value functions are Lipschitz continuous with respect to the tasks space.
	These theoretical results lead us to a value-transfer method for \lrl{}, which we use to build a PAC-MDP algorithm with improved convergence rate.
	Further, we show the method to experience no negative transfer with high probability.
	We illustrate the benefits of the method in \lrl{} experiments.
\end{abstract}

\section{Introduction}

Lifelong Reinforcement Learning (RL) is an online problem where an agent faces a series of RL tasks, drawn % from a (generally unknown) distribution.
sequentially.
Transferring knowledge from prior experience to speed up the resolution of new tasks is a key question in that setting \citep{lazaric2012transfer,taylor2009transfer}.
%Taking advantage of the knowledge gained in prior tasks is called transfer.
We elaborate on the intuitive idea that \emph{similar} tasks should allow a large amount of transfer.
An agent able to compute online a similarity measure between source tasks and the current target task could be able to perform transfer accordingly.
By measuring the amount of inter-task similarity, we design a novel method for value transfer, practically deployable in the online \lrl{} setting.
Specifically, we introduce a metric between MDPs and prove that the optimal Q-value function is Lipschitz continuous with respect to the MDP space.
This property makes it possible to compute a provable \ub{} on the optimal Q-value function of an unknown target task, given the learned optimal Q-value function of a source task.
Knowing this \ub{} accelerates the convergence of an \rmax{}-like algorithm \citep{brafman2002r}, relying on an optimistic estimate of the optimal Q-value function.
Overall, the proposed transfer method consists of computing online the distance between source and target tasks, deducing the \ub{} on the optimal Q value function of the source task and using this bound to accelerate learning.
Importantly, the method exhibits no negative transfer, \ie{}, it cannot cause performance degradation, as the computed \ub{} provably does not underestimate the optimal Q-value function.

Our contributions are as follows.
First, we study theoretically the Lipschitz continuity of the optimal Q-value function in the task space by introducing a metric between MDPs (Section~\ref{sec:lipschitz-continuity-results}).
Then, we use this continuity property to propose a value-transfer method based on a local distance between MDPs (Section~\ref{sec:transfer}).
Full knowledge of both MDPs is not required and the transfer is non-negative, which makes the method applicable online and safe.
In Section~\ref{sec:lrmax}, we build a PAC-MDP algorithm called \emph{Lipschitz \rmax{}}, applying this transfer method in the online \lrl{} setting.
We provide sample and computational complexity bounds and showcase the algorithm in \lrl{} experiments (Section~\ref{sec:experiments}).

\section{Background and Related Work}

Reinforcement Learning~(RL)~\citep{sutton2018reinforcement} is a framework for sequential decision making.
The problem is typically modeled as a Markov Decision Process (MDP) \citep{puterman2014markov} consisting of a 4-tuple $\langle \S, \A, R, T \rangle$, where $\S$ is a state space, $\A$ an action space, $R_s^a$ is the expected reward of taking action $a$ in state $s$ and $T_{s s'}^a$ is the transition probability of reaching state $s'$ when taking action $a$ in state $s$.
%For a complete overview, see \citet{sutton2018reinforcement,puterman2014markov}.
Without loss of generality, we assume $R_s^a \in [0, 1]$. % for all $s, a \in \S \times \A$.
%A policy $\pi$ is a decision rule mapping each state to an action or a distribution over actions. The value function $V^\pi$ is defined as the expected discounted reward of running a policy in a state $V^\pi(s_0) \triangleq \E_{\pi, R, T} \{ \sum_{t=0}^{\infty} \gamma^t R_{s_t}^{a_t} \}$ where $\gamma \in (0, 1)$ is a discount factor. The Q-function is the value of selecting an action in a state and applying a policy thereafter it satisfy the Bellman equation, $Q^\pi(s, a) = R_s^a + \gamma \sum_{s'} T_{s s'}^a V^\pi (s')$. Solving an MDP can reduce to finding the optimal Q-function and act greedily \wrt{} it. The optimal Q-function satisfies the optimality Bellman equation:
%\begin{equation}
%	Q^*(s, a) = \max_a \left[ R_s^a + \gamma \sum_{s'} T_{s s'}^a V^* (s') \right]
%\end{equation}
%For brevity, we shall use the notation $sa$ for a pair $(s,a) \in \S\times\A$.
Given a discount factor $\gamma \in [0,1)$, the expected cumulative return $\sum_t \gamma^t R_{s_t}^{a_t}$ obtained along a trajectory starting with state $s$ and action $a$ using policy $\pi$ in MDP $M$ is denoted by $Q^{\pi}_M(s,a)$ and called the Q-function.
The optimal Q-function $Q^*_M$ is the highest attainable expected return from $s, a$ and $V^*_M(s) = \max_{a \in \A} Q^*_M(s,a)$ is the optimal value function in $s$.
Notice that $R_{s}^{a} \leq 1$ implies $Q^*_M(s, a) \leq \frac{1}{1 - \gamma}$ for all $s, a \in \S \times \A$.
This maximum \ub{} is used by the \rmax{} algorithm as an optimistic initialization of the learned Q function.
A key point to reduce the sample complexity of this algorithm is to benefit from a tighter \ub{}, which is the purpose of our transfer method.

\lrl{}~\citep{silver2013lifelong,brunskill2014pac} is the problem of experiencing online a series of MDPs drawn from an unknown distribution. % $\D$.
Each time an MDP is sampled, a classical RL problem takes place where the agent is able to interact with the environment to maximize its expected return.
In this setting, it is reasonable to think that knowledge gained on previous MDPs could be re-used to improve the performance in new MDPs.
In this paper, we provide a novel method for such transfer by characterizing the way the optimal Q-function can evolve across tasks.
As commonly done \citep{wilson2007multi,brunskill2014pac,abel2018policy}, we restrict the scope of the study to the case where sampled MDPs share the same state-action space $\S \times \A$.
%These restricted sets of MDPs give a relevant insight on the question we try to answer and are of great importance in the RL literature.
For brevity, we will refer indifferently to MDPs, models or tasks, and write them $M=\langle R,T \rangle$.
%Thus, the \lrl{} setting consists of sequentially drawing models and solving them with interaction.

Using a metric between MDPs has the appealing characteristic of quantifying the amount of similarity between tasks, which intuitively should be linked to the amount of transfer achievable.
%& Transfer with bi-simulation metric
\citet{song2016measuring} define a metric based on the bi-simulation metric introduced by \citet{ferns2004metrics} and the Wasserstein metric~\citep{villani2008optimal}.
value transfer is performed between states with low bi-simulation distances.
However, this metric requires knowing both MDPs completely and is thus unusable in the \lrl{} setting where we expect to perform transfer before having learned the current MDP.
Further, the transfer technique they propose does allow negative transfer (see Appendix, Section~\ref{sec:example-negative-transfer}).
%&
\citet{carroll2005task} also define a value-transfer method based on a measure of similarity between tasks.
However, this measure is not computable online and thus not applicable to the \lrl{} setting.
%&
\citet{mahmud2013clustering} and \citet{brunskill2013sample} propose MDP clustering methods; respectively using a metric quantifying the regret of running the optimal policy of one MDP in the other MDP and the $\mathcal{L}_1$ norm between the MDP models.
An advantage of clustering is to prune the set of possible source tasks.
They use their approach for policy transfer, which differs from the value-transfer method proposed in this paper.
%&
\citet{ammar2014automated} learn the model of a source MDP and view the prediction error on a target MDP as a dissimilarity measure in the task space.
Their method makes use of samples from both tasks and is not readily applicable to the online setting considered in this paper.
%&
\citet{lazaric2008transfer} provide a practical method for sample transfer, computing a similarity metric reflecting the probability of the models to be identical.
Their approach is applicable in a batch RL setting as opposed to the online setting considered in this paper.
%\citep{fernandez2006probabilistic}
%\citep{konidaris2012transfer}
%&
The approach developed by \citet{sorg2009transfer} is very similar to ours in the sense that they prove bounds on the optimal Q-function for new tasks, assuming that both MDPs are known and that a soft homomorphism exists between the state spaces.
\citet{brunskill2013sample} also provide a method that can be used for Q-function bounding in multi-task RL.

\begin{figure}
	\centering
	\includegraphics[
	width=\linewidth,
	clip,
	trim={15 13 15 12}
	]{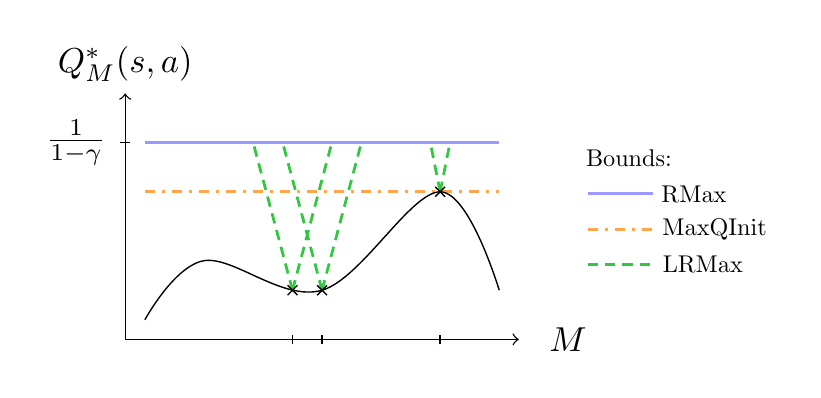}
	\caption{
		The optimal Q-value function represented for a particular $s, a$ pair across the MDP space.
		The \rmax{}, \maxqinit{} and \lrmax{} bounds are represented for three sampled MDPs.
	}
	\label{fig:bounds-illustration}
\end{figure}

%& MaxQInit
\citet{abel2018policy} present the MaxQInit algorithm, providing transferable bounds on the Q-function with high probability while preserving PAC-MDP guarantees~\citep{strehl2009reinforcement}. 
Given a set of solved tasks, they derive the probability that the maximum over the Q-values of previous MDPs is an \ub{} on the current task's optimal Q-function. 
This approach results in a method for non-negative transfer with high probability once enough tasks have been sampled.
%& Comparison
The method developed by \citet{abel2018policy} is similar to ours in two fundamental points:
first, a theoretical \ub{}s on optimal Q-values across the MDP space is built;
secondly, this provable \ub{} is used to transfer knowledge between MDPs by replacing the maximum $\frac{1}{1 - \gamma}$ bound in an \rmax{}-like algorithm, providing PAC guarantees.
The difference between the two approaches is illustrated in Figure~\ref{fig:bounds-illustration}, where the \maxqinit{} bound is the one developed by \citet{abel2018policy}, and the \lrmax{} bound is the one we present in this paper.
On this figure, the essence of the \lrmax{} bound is noticeable.
It stems from the fact that the optimal Q value function is locally Lipschitz continuous in the MDP space \wrt{} a specific pseudometric.
Confirming the intuition, close MDPs \wrt{} this metric have close optimal Q values.
It should be noticed that no bound is uniformly better than the other as intuited by Figure~\ref{fig:bounds-illustration}.
Hence, combining all the bounds results in a tighter \ub{} as we will illustrate in experiments (Section~\ref{sec:experiments}).
We first carry out the theoretical characterization of the Lipschitz continuity properties in the following section.
Then, we build on this result to propose a practical transfer method for the online \lrl{} setting. % via the \lrmax{} algorithm.

\section{Lipschitz Continuity of Q-Functions}
\label{sec:lipschitz-continuity-results}

The intuition we build on is that similar MDPs should have similar optimal Q-functions.
Formally, this insight can be translated into a continuity property of the optimal Q-function over the MDP space $\M$.
The remainder of this section mathematically formalizes this intuition that will be used in the next section to derive a practical method for value transfer. 
%We show in this section that the optimal Q-function is Lipschitz continuous with respect to a metric in the MDP space. 
%We introduce a notion of quasi-Lipschitz continuity of $Q^*$ across tasks in this section.
To that end, we introduce a local pseudometric characterizing the distance between the models of two MDPs at a particular state-action pair. 
A reminder and a detailed discussion on the metrics used herein can be found in the Appendix, Section~\ref{sec:app:local-mdps-distance-discussion}.
\begin{definition}%
	\label{def:pseudo-metric-between-models}
	Given two tasks $M = \langle R, T \rangle$, $\bar{M} = \langle \bar{R}, \bar{T} \rangle$, and a function $f: \S \rightarrow \mathbb{R}^+$, we define the \emph{pseudometric between models} at $(s, a) \in \SA$ \wrt{} $f$ as:
	\begin{equation}
	\label{eq:local-model-pseudo-metric}
	\modpm{s}{a}{f}{M}{\bar{M}} \triangleq |R_s^a - \bar{R}_s^a| + \sum_{s' \in \S} f(s') |T_{s s'}^a - \bar{T}_{s s'}^a|.
	\end{equation}
\end{definition}%
This pseudometric is relative to a positive function $f$.
We implicitly cast this definition in the context of discrete state spaces.
The extension to continuous spaces is straightforward but beyond the scope of this paper.
For the sake of clarity in the remainder of this study, we introduce% $\modiv{s}{a}{\gamma V^*_{\bar{M}}}{M}{\bar{M}} \eqdef \modpm{s}{a}{\gamma V^*_{\bar{M}}}{M}{\bar{M}}$,
\begin{equation*}
\modiv{s}{a}{\gamma V^*_{\bar{M}}}{M}{\bar{M}} \eqdef \modpm{s}{a}{\gamma V^*_{\bar{M}}}{M}{\bar{M}},
\end{equation*}
corresponding to the pseudometric between models with the particular choice of $f = \gamma V^*_{\bar{M}}$.
From this definition stems the following pseudo-Lipschitz continuity result.

\begin{proposition}[Local pseudo-Lipschitz continuity]
	\label{proposition:local-lipschitz-continuity}
	For two MDPs $M, \bar{M}$, for all $(s, a) \in \SA$,
	\begin{equation}
	\label{eq:local-lipschitz-continuity}
	\absnorm{\Q{*}{M}(s, a) - \Q{*}{\bar{M}}(s, a)} \leq \mdppm{s}{a}{M}{\bar{M}},
	\end{equation}
	with the \emph{local MDP pseudometric} $\mdppm{s}{a}{M}{\bar{M}} \eqdef \min \left\{ \mdpdiv{s}{a}{M}{\bar{M}}, \mdpdiv{s}{a}{\bar{M}}{M} \right\}$, and the \emph{local MDP dissimilarity} $\mdpdiv{s}{a}{M}{\bar{M}}$ is the unique solution to the following \fp{} equation for $d_{s a}$:
	\begin{equation}
	\label{eq:asym-mdp-pseudo-distance}
	d_{s a} = \modiv{s}{a}{\gamma V^*_{\bar{M}}}{M}{\bar{M}} + \gamma \sum_{s' \in \S} T_{s s'}^a \max_{a' \in \A} d_{s' a'}, \forall s, a.
	\end{equation}
\end{proposition}

All the proofs of the paper can be found in the Appendix.
This result establishes that the distance between the optimal Q-functions of two MDPs at $(s, a) \in \S \times \A$ is controlled by a local dissimilarity between the MDPs. 
The latter follows a fixed-point equation (Equation \ref{eq:asym-mdp-pseudo-distance}), which can be solved by Dynamic Programming~(DP) \citep{bellman1957dynamic}. % \eg{}, in a Value Iteration procedure.
%For the result to hold, the function $f : s \mapsto \gamma V^*_{\bar{M}}(s) \in \R^+$ was selected in the definition of the model's pseudometric (Equation~\ref{eq:local-model-pseudo-metric}).
Note that, although the local MDP dissimilarity $\mdpdiv{s}{a}{M}{\bar{M}}$ is asymmetric, $\mdppm{s}{a}{M}{\bar{M}}$ \emph{is} a pseudometric, hence the name \emph{pseudo-Lipschitz continuity}.
%The two MDPs are not used in the same way and both roles are interchangeable.
Notice that the policies in Equation~\ref{eq:local-lipschitz-continuity} are the optimal ones for the two MDPs and thus are different.
%	This pseudo-Lipschitz continuity result is to be distinguished from other frameworks of the literature that \textit{assume} the continuity of the reward and transition models \wrt{} $\SA$ \citep{rachelson2010locality,pirotta2015policy,asadi2018lipschitz}. In our case, this is not an assumption but a mathematical result stemming from Definition~\ref{def:pseudo-metric-between-models}.%
Proposition~\ref{proposition:local-lipschitz-continuity} is a mathematical result stemming from Definition~\ref{def:pseudo-metric-between-models} and should be distinguished from other frameworks of the literature that \textit{assume} the continuity of the reward and transition models \wrt{} $\SA$ \citep{rachelson2010locality,pirotta2015policy,asadi2018lipschitz}.%
This result establishes that the optimal Q-functions of two close MDPs, in the sense of Equation \ref{eq:local-model-pseudo-metric}, are themselves close to each other.
Hence, given $Q^*_{\bar{M}}$, the function
\begin{equation}
s, a \mapsto Q^*_{\bar{M}}(s, a) + \mdppm{s}{a}{M}{\bar{M}}
\label{eq:local-ub}
\end{equation}
can be used as an \ub{} on $Q^*_{M}$ with $M$ an unknown MDP.
This is the idea on which we construct a computable and transferable \ub{} in Section~\ref{sec:transfer}.
In Figure~\ref{fig:bounds-illustration}, the \ub{} of Equation~\ref{eq:local-ub} is represented by the \lrmax{} bound.
Noticeably, we provide a global pseudo-Lipschitz continuity property, along with similar results for the optimal value function $V^*_M$ and the value function of a fixed policy.
As these results do not directly serve the purpose of this article, we report them in the Appendix, Section~\ref{sec:app:other-results}.

\section{Transfer Using the Lipschitz Continuity}
\label{sec:transfer}

A purpose of value transfer, when interacting online with a new MDP, is to initialize the value function and drive the exploration to accelerate learning.
%A purpose of value transfer is to benefit from a value function that drives the exploration to accelerate learning.
We aim to exploit value transfer in a method guaranteeing three conditions:\\
%\textbf{C1.} the resulting algorithm is PAC-MDP; %~\citep{strehl2009reinforcement};
%\textbf{C2.} the transfer accelerates learning;
%\textbf{C3.} the transfer is non-negative.
\indent C1. the resulting algorithm is PAC-MDP;\\
\indent C2. the transfer accelerates learning;\\
\indent C3. the transfer is non-negative.\\
%\begin{itemize}
%	\setlength\itemsep{0cm}
%	\item[] C1. the resulting algorithm is PAC-MDP; %~\citep{strehl2009reinforcement};
%	\item[] C2. the transfer accelerates learning;
%	\item[] C3. the transfer is non-negative.
%\end{itemize}
To achieve these conditions, we first present a transferable \ub{} on $Q^*_M$ in Section~\ref{sec:transferable-ub}.
This \ub{} stems from the Lipschitz continuity result of Proposition~\ref{proposition:local-lipschitz-continuity}.
%s presented in Section~\ref{sec:lipschitz-continuity-results}.
Then, we propose a practical way to \emph{compute} this \ub{} in Section~\ref{sec:computable-ub}.
Precisely, we propose a surrogate bound that can be calculated online in the \lrl{} setting, without having explored the source and target tasks completely.
Finally, we implement the method in an algorithm described in Section~\ref{sec:lrmax}, and demonstrate formally that it meets conditions C1, C2 and C3.
Improvements are discussed in Section~\ref{sec:improving-lrmax}.

\subsection{A Transferable \UB{} on $Q^*_M$}
\label{sec:transferable-ub}

%& Induced Lipschitz bound
From Proposition~\ref{proposition:local-lipschitz-continuity}, one can naturally define a local \ub{} on the optimal Q-function of an MDP given the optimal Q-function of another MDP.
\begin{definition}
	\label{def:lipschitz-bound}
	Given two tasks $M$ and $\bar{M}$, for all $(s, a) \in \S \times \A$, the \emph{Lipschitz \ub{} on $Q^*_M$ induced by $Q^*_{\bar{M}}$} is defined as $U_{\bar{M}}(s, a) \geq Q^*_M(s,a)$ with:
	\begin{equation}
	U_{\bar{M}}(s, a) \triangleq Q^*_{\bar{M}}(s, a) + \mdppm{s}{a}{M}{\bar{M}}.
	\label{eq:lipschitz-bound}
	\end{equation}
\end{definition}
The \emph{optimism in the face of uncertainty} principle leads to considering that the long-term expected return from any state is the $\frac{1}{1-\gamma}$ maximum return, unless proven otherwise. 
Particularly, the \rmax{} algorithm \citep{brafman2002r}, explores an MDP so as to shrink this \ub{}.
\rmax{} is a model-based, online RL algorithm with PAC-MDP guarantees \citep{strehl2009reinforcement}, meaning that convergence to a near-optimal policy is guaranteed in a polynomial number of missteps with high probability.
It relies on an optimistic model initialization that yields an optimistic \ub{} $U$ on the optimal Q-function, then acts greedily \wrt{} $U$.
By default, it takes the maximum value $U(s, a) = \frac{1}{1 - \gamma}$, but any tighter \ub{} is admissible.
Thus, shrinking $U$ with Equation~\ref{eq:lipschitz-bound} is expected to improve the learning speed or sample complexity for new tasks in \lrl{}.

In \rmax{}, during the resolution of a task $M$, $\SA$ is split into a subset of known state-action pairs $K$ and its complement $K^c$ of unknown pairs.
A state-action pair is known if the number of collected reward and transition samples allows estimating an $\epsilon$-accurate model in $\mathcal{L}_1$-norm with probability higher than $1 - \delta$.
We refer to $\epsilon$ and $\delta$ as the \emph{\rmax{} precision parameters}.
This results in a threshold $n_{known}$ on the number of visits $n(s,a)$ to a pair $s,a$ that are necessary to reach this precision.
%& Combining all the Lipschitz bounds
Given the experience of a set of $m$ MDPs $\bar{\M} = \{ \bar{M}_1, \ldots, \bar{M}_m \}$, we define the total bound as the minimum over all the induced Lipschitz bounds.
\begin{proposition}%
	\label{proposition:total-ub}%
	Given a partially known task $M = \langle R, T \rangle$, the set of known state-action pairs $K$, and the set of Lipschitz bounds on $Q^*_M$ induced by previous tasks $\left\{ U_{\bar{M}_1}, \ldots, U_{\bar{M}_m} \right\}$, the function $Q$ defined below is an \ub{} on $Q^*_M$ for all $s, a \in \S \times \A$.
	\begin{equation}%
		Q(s, a) \eqdef
		\begin{cases}
		R_s^a + \gamma \sum\limits_{s' \in \S} T_{s s'}^a \max\limits_{a' \in \A} Q(s', a') \\
		\hfill \text{if } (s, a) \in K, \\
		U(s,a) \text{ otherwise,}
		\end{cases}
		\label{eq:total-ub}%
	\end{equation}%
	with $U(s,a)=\min \left\{ \frac{1}{1 - \gamma}, U_{\bar{M}_1}(s, a), \ldots, U_{\bar{M}_m}(s, a) \right\}$.
\end{proposition}%
Commonly in \rmax{}, Equation~\ref{eq:total-ub} is solved to a precision $\epsilon_Q$ via Value Iteration.
This yields a function $Q$ that is a valid heuristic (bound on $Q^*_M$) for the exploration of MDP $M$.

\subsection{A Computable \UB{} on $Q^*_M$}
\label{sec:computable-ub}

The key issue addressed in this section is how to actually compute $U(s,a)$, particularly when both source and target tasks are partially explored.
Consider two tasks $M$ and $\bar{M}$, on which vanilla \rmax{} has been applied, yielding the respective sets of known state-action pairs $K$  and $\bar{K}$, along with the learned models $\hat{M} = \langle \hat{T}, \hat{R} \rangle$ and $\hat{\bar{M}} = \langle \hat{\bar{T}}, \hat{\bar{R}} \rangle$, and the \ub{}s $Q$ and $\bar{Q}$ respectively on $Q^*_M$ and $Q^*_{\bar{M}}$.
%\footnote{
Notice that, if $K = \emptyset$, then $Q(s,a)=\frac{1}{1-\gamma}$ for all $s, a$ pairs.
Conversely, if $K^c = \emptyset$, $Q$ is an $\epsilon$-accurate estimate  of $Q^*_{M}$  in $\mathcal{L}_1$-norm with high probability.
Equation~\ref{eq:total-ub} allows the transfer of knowledge from $\bar{M}$ to $M$ if $U_{\bar{M}}(s,a)$ can be computed.
Unfortunately, the true model and optimal value functions, necessary to compute $U_{\bar{M}}$, are \emph{partially} known (see Equation~\ref{eq:lipschitz-bound}).
Thus, we propose to compute a looser \ub{} based on the learned models and value functions.
First, we provide an \ub{} $\modivhat{s}{a}{M}{\bar{M}}$ on $\modiv{s}{a}{\gamma \V{*}{\bar{M}}}{M}{\bar{M}}$ (Definition~\ref{def:pseudo-metric-between-models}).

\begin{proposition}
	\label{proposition:ub-model-pseudo-metric}
	Given two tasks $M$, $\bar{M}$ and respectively $K$, $\bar{K}$ the subsets of $\S \times \A$ where their models are known with accuracy $\epsilon$ in $\mathcal{L}_1$-norm with probability at least $1 - \delta$,
	%We have that:
	\begin{equation*}
	\pr \left( \modivhat{s}{a}{M}{\bar{M}} \geq \modiv{s}{a}{\gamma V^*_{\bar{M}}}{M}{\bar{M}} \right) \geq 1 - \delta
	\end{equation*}
	with $\modivhat{s}{a}{M}{\bar{M}}$ the \emph{\ub{} on the pseudometric between models} defined below for $B = \epsilon \left(1 + \gamma \max_{s'} \bar{V}(s') \right)$.
	\begin{align}%
	& \modivhat{s}{a}{M}{\bar{M}} \; \eqdef \; \nonumber \\
	& \begin{cases}%
	\modpm{s}{a}{\gamma \bar{V}}{\hat{M}}{\hat{\bar{M}}} + 2 B
	& \text{if } (s, a) \in K \cap \bar{K} \\
	\max\limits_{\bar{\mu} \in \M} \modpm{s}{a}{\gamma \bar{V}}{\hat{M}}{\bar{\mu}} + B
	& \text{if } (s, a) \in K \cap \bar{K}^c \\
	\max\limits_{\mu \in \M} \modpm{s}{a}{\gamma \bar{V}}{\mu}{\hat{\bar{M}}} + B
	& \text{if } (s, a) \in K^c \cap \bar{K} \\
	\max\limits_{\mu, \bar{\mu} \in \M^2} \modpm{s}{a}{\gamma \bar{V}}{\mu}{\bar{\mu}}
	& \text{if } (s, a) \in K^c \cap \bar{K}^c
	\end{cases}%
	\label{eq:local-distance-ub}%
	\end{align}%
%	\begin{equation}%
%	\modivhat{s}{a}{M}{\bar{M}} \eqdef%
%	\begin{cases}%
%	\modpm{s}{a}{\gamma \bar{V}}{\hat{M}}{\hat{\bar{M}}} + 2 B
%	& \text{if } s, a \in K \cap \bar{K} \\
%	\max\limits_{\bar{\mu} \in \M} \modpm{s}{a}{\gamma \bar{V}}{\hat{M}}{\bar{\mu}} + B
%	& \text{if } s, a \in K \cap \bar{K}^c \\
%	\max\limits_{\mu \in \M} \modpm{s}{a}{\gamma \bar{V}}{\mu}{\hat{\bar{M}}} + B
%	& \text{if } s, a \in K^c \cap \bar{K} \\
%	\max\limits_{\mu, \bar{\mu} \in \M^2} \modpm{s}{a}{\gamma \bar{V}}{\mu}{\bar{\mu}}
%	& \text{if } s, a \in K^c \cap \bar{K}^c
%	\end{cases}%
%	\label{eq:local-distance-ub-bis}%
%	\end{equation}%
%	where $B = \epsilon \left(1 + \gamma \max_{s'} \bar{V}(s') \right)$.
\end{proposition}

Importantly, this \ub{} $\modivhat{s}{a}{M}{\bar{M}}$ can be calculated analytically (see Appendix, Section~\ref{sec:app:analytical-calculation-dmodelhat}).
This makes $\modivhat{s}{a}{M}{\bar{M}}$ usable in the online \lrl{} setting, where already explored tasks may be partially learned, and little knowledge has been gathered on the current task.
The magnitude of the $B$ term is controlled by $\epsilon$.
In the case where no information is available on the maximum value of $\bar{V}$, we have that $B = \frac{\epsilon}{1 - \gamma}$.
$\epsilon$ measures the accuracy with which the tasks are known: the smaller $\epsilon$, the tighter the $B$ bound.
Note that $\bar{V}$ is used as an \ub{} on the true $V^*_{\bar{M}}$.
In many cases, $\max_{s'} V^*_{\bar{M}}(s') \leq \frac{1}{1-\gamma}$; \eg{} for stochastic shortest path problems, which feature rewards only upon reaching terminal states, we have that $\max_{s'} V^*_{\bar{M}}(s')=1$ and thus $B=(1+\gamma)\epsilon$ is a tighter bound for transfer.
Combining $\modivhat{s}{a}{M}{\bar{M}}$ and Equation \ref{eq:asym-mdp-pseudo-distance}, one can derive an \ub{} $\mdpdivhat{s}{a}{M}{\bar{M}}$ on $\mdpdiv{s}{a}{M}{\bar{M}}$, detailed in Proposition~\ref{proposition:asym-mdp-pseudo-distance-upperbound}.
\begin{proposition}
	\label{proposition:asym-mdp-pseudo-distance-upperbound}
	Given two tasks $M, \bar{M} \in \M$, $K$ the set of state-action pairs where $(\Reward, \transition)$ is known with accuracy $\epsilon$ in $\mathcal{L}_1$-norm with probability at least $1 - \delta$.
	If $\gamma (1 + \epsilon) < 1$, the solution $\mdpdivhat{s}{a}{M}{\bar{M}}$ of the following \fp{} equation on $\hat{d}_{s a}$ (for all $s, a \in \SA$) is an \ub{} on $\mdpdiv{s}{a}{M}{\bar{M}}$ with probability at least $1 - \delta$:
	\begin{align}
	\label{eq:asym-mdp-pseudo-distance-upperbound}
	& \hat{d}_{s a} = \modivhat{s}{a}{M}{\bar{M}} \; + \\
	&
	\begin{cases}
		\gamma \left( \sum\limits_{s' \in \S} \hat{T}_{s s'}^a \max\limits_{a' \in \A} \hat{d}_{s' a'} + \epsilon \max\limits_{s', a' \in \S \times \A} \hat{d}_{s' a'} \right) \text{ if } s, a \in K, \\
		\gamma \max\limits_{s', a' \in \S \times \A} \hat{d}_{s' a'}
		\text{ otherwise.}
		\end{cases} \nonumber
	\end{align}
\end{proposition}
Similarly as in Proposition~\ref{proposition:ub-model-pseudo-metric}, the condition ${\gamma (1 + \epsilon) < 1}$ illustrates the fact that for a large return horizon (large $\gamma$), a high accuracy (small $\epsilon$) is needed for the bound to be computable.
Eventually, a computable \ub{} on $Q^*_M$ given $\bar{M}$ with high probability is given by
%\begin{equation}
%\hat{U}_{\bar{M}}(s,a) = \bar{Q}(s,a) + \min \left\{ \mdpdivhat{s}{a}{M}{\bar{M}}, \mdpdivhat{s}{a}{\bar{M}}{M} \right\}.
%\label{eq:ub-on-lipschitz-bound}
%\end{equation}
\begin{align}
\nonumber \hat{U}_{\bar{M}}(s,a) & = \bar{Q}(s,a)\\%
& + \min \left\{ \mdpdivhat{s}{a}{M}{\bar{M}}, \mdpdivhat{s}{a}{\bar{M}}{M} \right\}.
\label{eq:ub-on-lipschitz-bound}
\end{align}
The associated \ub{} on $U(s,a)$ (Equation \ref{eq:total-ub}) given the set of previous tasks $\bar{\M}=\{ \bar{M}_i \}_{i=1}^m$ is defined by
\begin{equation}
\hat{U}(s,a) = \min \left\{ {\textstyle \frac{1}{1 - \gamma}}, \hat{U}_{\bar{M}_1}(s, a), \ldots, \hat{U}_{\bar{M}_m}(s, a) \right\}.
\label{eq:ub-on-u}
\end{equation}
This \ub{} can be used to transfer knowledge from a partially solved source task to a target task.
If $\hat{U} (s, a) \leq \frac{1}{1 - \gamma}$ on a subset of $\SA$, then the convergence rate can be improved.
As complete knowledge of both tasks is not needed to compute the \ub{}, it can be applied online in the \lrl{} setting.
In the next section, we explicit an algorithm that leverages this value-transfer method.

\subsection{Lipschitz \rmax{} Algorithm}
\label{sec:lrmax}

In \lrl{}, MDPs are encountered sequentially. 
Applying \rmax{} to task $M$ yields the set of known state-action pairs $K$, the learned models $\hat{\transition}$ and $\hat{\Reward}$, and the \ub{} $Q$ on $\Q{*}{M}$. 
Saving this information when the task changes allows computing the \ub{} of Equation~\ref{eq:ub-on-u} for the new target task, and using it to shrink the optimistic heuristic of \rmax{}.
This computation effectively transfers value functions between tasks based on task similarity.
As the new task is explored online, the task similarity is progressively assessed with better confidence, refining the values of $\modivhat{s}{a}{M}{\bar{M}}$, $\mdpdivhat{s}{a}{M}{\bar{M}}$ and eventually $\hat{U}$, allowing for more efficient transfer where the task similarity is appraised.
%Practically, the method presented in Section~\ref{sec:computable-ub} is used to compute the local MDP distance of Equation~\ref{eq:asym-mdp-pseudo-distance}.
The resulting algorithm, Lipschitz \rmax{} (\lrmax{}), is presented in Algorithm~\ref{alg:lrmax}.
To avoid ambiguities with $\bar{\M}$, we use $\hat{\M}$ to store learned features ($\hat{T}$, $\hat{R}$, $K$, $Q$) about previous MDPs.
\begin{algorithm2e}[t]
	\DontPrintSemicolon
	%	\begin{multicols}{2}
	%	Function \lrmax{}$()$:\;
	Initialize $\hat{\M} = \emptyset$.\;
	\For{each newly sampled MDP $M$}{
		Initialize $Q(s,a)=\frac{1}{1-\gamma}, \forall s,a$, and $K=\emptyset$\;
		Initialize $\hat{T}$ and $\hat{R}$ (\rmax{} initialization)\;
		$Q \leftarrow$ UpdateQ$(\hat{\M},\hat{T},\hat{R})$\;
		\For{$t \in [1,$ max number of steps$]$}{
			$s=$ current state, % \;
			$a = \arg\max\limits_{a'} Q(s, a')$\;
			Observe reward $r$ and next state $s'$ \;
			$n(s,a) \leftarrow n(s,a)+1$ \;
			\If{$n(s,a)<n_{known}$}{
				Store $(s,a,r,s')$\;% for future model updates\;
			}
			\If{$n(s,a)=n_{known}$}{
				Update $K$ and $(\hat{T}_{ss'}^a, \hat{R}_s^a)$ (learned model) \;
				$Q \leftarrow$ UpdateQ$(\hat{\M},\hat{T},\hat{R})$
			}
		}
		Save $\hat{M} = \left(\hat{T},\hat{R},K,Q\right)$ in $\hat{\M}$
	}
	Function UpdateQ$(\hat{\M},\hat{T},\hat{R})$:\;
	\For{${\bar{M}} \in \bar{\M}$}{
		Compute $\modivhat{s}{a}{M}{\bar{M}}$, $\modivhat{s}{a}{\bar{M}}{M}$ (Eq. \ref{eq:local-distance-ub})\;
		Compute $\mdpdivhat{s}{a}{M}{\bar{M}}$, $\mdpdivhat{s}{a}{\bar{M}}{M}$ (DP on Eq. \ref{eq:asym-mdp-pseudo-distance-upperbound})\;
		Compute $\hat{U}_{\bar{M}}$ (Eq. \ref{eq:ub-on-lipschitz-bound})\;
	}
	Compute $\hat{U}$ (Eq.~\ref{eq:ub-on-u})\;
	Compute and return $Q$ (DP on Eq. \ref{eq:total-ub} using $\hat{U}$)
	%	\end{multicols}
	\;
	\caption{Lipschitz \rmax{} algorithm}
	\label{alg:lrmax}
\end{algorithm2e}
In a nutshell, the behavior of \lrmax{} is precisely that of \rmax{}, but with a tighter admissible heuristic $\hat{U}$ that becomes better as the new task is explored (while this heuristic remains constant in vanilla \rmax{}).
%Lipschitz \rmax{} measures the local distance to each of the MDPs in memory to refine its \ub{} on the the current MDP's optimal Q-function. As new samples are gathered, the estimate of the Lipschitz bounds of Equation~\ref{eq:lipschitz-bound} decreases and the total \ub{} of Proposition~\ref{proposition:total-ub} becomes more accurate.
\lrmax{} is PAC-MDP (Condition~C1) as stated in Propositions~\ref{proposition:sample-complexity} and~\ref{proposition:computational-complexity} below.
With $S=|\S|$ and $A=|\A|$, the sample complexity of vanilla \rmax{} is $\tilde{\mathcal{O}} ( S^2 A / (\epsilon^3 (1 - \gamma)^3) )$, which is improved by \lrmax{} in Proposition~\ref{proposition:sample-complexity} and meets Condition~C2.
Finally, $\hat{U}$ is a provable \ub{} with high probability on $Q^*_M$, which avoids negative transfer and meets Condition~C3.
\begin{proposition}[Sample complexity \citep{strehl2009reinforcement}]
	\label{proposition:sample-complexity}
	With probability $1 - \delta$, the greedy policy \wrt{} $Q$ computed by \lrmax{} achieves an $\epsilon$-optimal return in MDP $M$ after 
	\begin{equation*}
	\tilde{\mathcal{O}} \left( \frac{S |\{ s, a \in \S \times \A \mid \hat{U}(s, a) \geq V^*_{M}(s) - \epsilon \} |}{\epsilon^3 (1 - \gamma)^3} \right)
	\end{equation*}
	samples (when logarithmic factors are ignored), with $\hat{U}$ defined in Equation~\ref{eq:ub-on-u} a non-static, decreasing quantity, \ub{}ed by $\frac{1}{1-\gamma}$.
\end{proposition}
Proposition ~\ref{proposition:sample-complexity} shows that the sample complexity of \lrmax{} is no worse than that of \rmax{}.
Consequently, in the worst case, \lrmax{} performs as badly as learning from scratch, which is to say that the transfer method is not negative as it cannot degrade the performance.
\begin{proposition}[Computational complexity]
	\label{proposition:computational-complexity}
	The total computational complexity of \lrmax{} (Algorithm~\ref{alg:lrmax}) is
	\begin{equation*}
	\tilde{\bigo} \left( \ntimesteps + \frac{\nS^3 \nA^2 N}{(1 - \gamma)} \ln \left( \frac{1}{\epsilon_Q (1 - \gamma)} \right) \right)
	\end{equation*}
	with $\ntimesteps$ the number of interaction steps, $\epsilon_Q$ the precision of value iteration and $N$ the number of source tasks.
\end{proposition}

\subsection{Refining the \lrmax{} Bounds}% with maximum model distance}
\label{sec:improving-lrmax}

\lrmax{} relies on bounds on the local MDP dissimilarity (Equation~\ref{eq:asym-mdp-pseudo-distance-upperbound}).
%The quality of the Lipschitz bound on $Q^*_M$ depends on the quality of those estimates and can be improved accordingly.
The quality of the Lipschitz bound on $Q^*_M$  can be improved according to the quality of those estimates.
We discuss two methods to provide finer estimates.

\textbf{Refining with prior knowledge.}
First, from the definition of $\modiv{s}{a}{\gamma V^*_{\bar{M}}}{M}{\bar{M}}$, it is easy to show that this pseudometric between models is always \ub{}ed by $\frac{1+\gamma}{1-\gamma}$.
However, in practice, the tasks experienced in a \lrl{} experiment might not cover the full span of possible MDPs $\M$ and may systematically be closer to each other than $\frac{1 + \gamma}{1 - \gamma}$.
For instance, the distance between two games in the Arcade Learning Environment (ALE) \citep{bellemare2013ale}, is smaller than the maximum distance between any two MDPs defined on the common state-action space of the ALE (extended discussion in Appendix, Section \ref{sec:app:dmax}).
\begin{figure}
	\centering
	\includegraphics[
	width=\linewidth,
	clip,
	trim={10 0 7 0}
	]{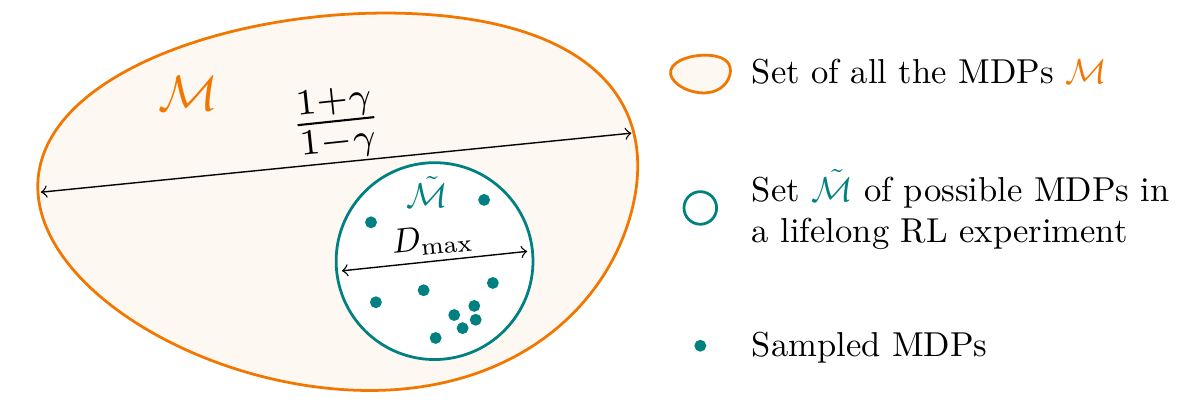}
	\caption{%
		Illustration of the prior knowledge on the maximum pseudo-distance between models for a particular $s, a$ pair.%
		%		The \textbf{\color{red!50!orange!75!black} red} diamond denotes the whole set $\M$ of MDPs, represented in 2D for the sake of illustration.
		%		The \textbf{\color{teal} blue} circle denotes the set of MDPs from which tasks are actually sampled.
		%		The \textbf{\color{teal} blue} bullets \textbf{\color{teal} $\circ$} represent sampled tasks.
	}
	\label{fig:prior-illustration}
\end{figure}
Let us note $\tilde{\M} \subset \M$ the set of possible MDPs for a particular \lrl{} experiment.
Let $D_{\max}(s, a) \eqdef \max_{M, \bar{M} \in \tilde{\M}^2} \left( \modiv{s}{a}{\gamma V^*_{\bar{M}}}{M}{\bar{M}} \right)$ be the \emph{maximum model pseudo-distance} at a particular $s, a$ pair on the subset $\tilde{\M}$.
\emph{Prior knowledge} might indicate a smaller \ub{} for $D_{\max}(s, a)$ than $\frac{1 + \gamma}{1 - \gamma}$.
We will note such an \ub{} $\prior$, considered valid for all $s, a$ pairs, \ie{}, such that $\prior \geq \max\limits_{s, a, M, \bar{M} \in \SA \times \tilde{\M}^2} \left( \modiv{s}{a}{\gamma V^*_{\bar{M}}}{M}{\bar{M}} \right)$.
%\begin{equation*}
%	\prior \geq \max\limits_{s, a, M, \bar{M} \in \SA \times \tilde{\M}^2} \left( \modiv{s}{a}{\gamma V^*_{\bar{M}}}{M}{\bar{M}} \right).
%\end{equation*}
In a \lrl{} experiment, $\prior$ can be seen as a rough estimate of the maximum model discrepancy an agent may encounter.
Figure~\ref{fig:prior-illustration} illustrates the relative importance of $\prior$ \vs{} $\frac{1 + \gamma}{1 - \gamma}$.
Solving Equation~\ref{eq:asym-mdp-pseudo-distance-upperbound} boils down to accumulating $\modivhat{s}{a}{M}{\bar{M}}$ values in $\mdpdivhat{s}{a}{M}{\bar{M}}$.
Hence, reducing a $\modivhat{s}{a}{M}{\bar{M}}$ estimate in a single $s, a$ pair actually reduces $\mdpdivhat{s}{a}{M}{\bar{M}}$ in \emph{all} $s, a$ pairs.
Thus, replacing $\modivhat{s}{a}{M}{\bar{M}}$ in Equation \ref{eq:asym-mdp-pseudo-distance-upperbound} by $\min \{ \prior, \modivhat{s}{a}{M}{\bar{M}} \}$, provides a smaller \ub{} $\mdpdivhat{s}{a}{M}{\bar{M}}$ on $\mdpdiv{s}{a}{M}{\bar{M}}$, and thus a smaller $\hat{U}$ which allows transfer if it is less than $\frac{1}{1-\gamma}$.
Consequently, the knowledge of such a bound $\prior$ can make a difference between successful and unsuccessful transfer, even if its value is of little importance.
Conversely, setting a value for $\prior$ quantifies the distance between MDPs where transfer is efficient.

\textbf{Refining by learning the maximum distance.}
The value of $D_{\max}(s, a)$ can be estimated online for each $s, a$ pair, discarding the hypothesis of available prior knowledge. 
We propose to use an empirical estimate of the maximum model distance at $s, a$: $\hat{D}_{\max}(s, a) \triangleq \max_{M, \bar{M} \in \hat{\M}^2} \{ \modivhat{s}{a}{M}{\bar{M}} \}$, with $\hat{\M}$ the set of explored tasks.
The pitfall of this approach is that, with few explored tasks, $\hat{D}_{\max}(s, a)$ could underestimate $D_{\max}(s, a)$.
Proposition~\ref{proposition:max-distance-estimate} provides a lower bound on the probability that $\hat{D}_{\max}(s, a) + \epsilon$ does not underestimate $D_{\max}(s, a)$, depending on the number of sampled tasks.% in $\hat{\M}$.
\begin{proposition}%
	\label{proposition:max-distance-estimate}%
	Consider an algorithm producing $\epsilon$-accurate model estimates $\modivhat{s}{a}{M}{\bar{M}}$ for a subset $K$ of $\SA$ after interacting with any two MDPs $M, \bar{M} \in \M$.
	Assume $\modivhat{s}{a}{M}{\bar{M}} \geq \modiv{s}{a}{\gamma V^*_{\bar{M}}}{M}{\bar{M}}$ for any $s, a \notin K$.
	For all $s, a \in \SA$, $\delta \in (0, 1]$, after sampling $m$ tasks, if $m$ is large enough to verify $2 (1 - \pmin)^m - (1 - 2\pmin)^m \leq \delta$,% then,
	\begin{equation*}%
		\Pr \left( \hat{D}_{\max}(s, a) + \epsilon \geq \prior(s, a) \right) \geq 1 - \delta \mthspc .
	\end{equation*}%
\end{proposition}%
This result indicates when $\hat{D}_{\max}(s, a) + \epsilon$ \ub{}s $D_{\max}(s, a)$ with high probability.
In such a case, $\modivhat{s}{a}{M}{\bar{M}}$ of Equation \ref{eq:asym-mdp-pseudo-distance-upperbound} can be replaced by $\min \{ \hat{D}_{\max}(s, a) + \epsilon, \modivhat{s}{a}{M}{\bar{M}} \}$ to tighten the bound on $\mdpdiv{s}{a}{M}{\bar{M}}$.
Assuming a lower bound $\p$ on the sampling probability of a task implies that $\M$ is finite and is seen as a non-adversarial task sampling rule \citep{abel2018policy}.

\section{Experiments}
\label{sec:experiments}

The experiments reported here\footnote{%
	\if\isanonymous1%
	Link to open-source code omitted for anonymity.
	\else%
	Code available at https://github.com/SuReLI/llrl
	\fi%
} illustrate how
%1) \lrmax{} allows for early performance increase in \lrl{} by effectively transferring knowledge between tasks;
%%2) the Lipschitz bound (Equation~\ref{eq:ub-on-lipschitz-bound}) improves the sample complexity compared to \rmax{}. % by providing a tighter \ub{} on $Q^*$.
%2) the Lipschitz bound (Equation~\ref{eq:ub-on-lipschitz-bound}) provides a tighter \ub{} on $Q^*$. %, improving the sample complexity compared to \rmax{}.
the Lipschitz bound (Equation~\ref{eq:ub-on-lipschitz-bound}) provides a tighter \ub{} on $Q^*$, improving the sample complexity of \lrmax{} compared to \rmax{}, and making the transfer of inter-task knowledge effective. % by effectively transferring knowledge between tasks.
Graphs are displayed with 95\% confidence intervals.
For information in line with the Machine Learning Reproducibility Check-list~\citep{mlreproducibility2019} see the Appendix, Section~\ref{sec:app:reproducibility-checklist}.

%\begin{figure} % [!htp]
%	\vspace{-6em}
%	\subfloat[Average discounted return vs. tasks]{
%		\begin{minipage}[c][1\width]{0.5\textwidth}
%			\centering
%			\includegraphics[width=\textwidth,clip,trim={0 0 20 10}]{img/tight_discounted_return_vs_task.pdf}
%			\label{fig:discounted_return_vs_task}
%			\vspace{-7em}
%	\end{minipage}}
%	\hfill
%	\subfloat[Average discounted return vs. episodes]{
%		\begin{minipage}[c][1\width]{0.5\textwidth}
%			\centering
%			\includegraphics[width=\textwidth,clip,trim={0 0 40 40}]{img/tight_discounted_return_vs_episode.pdf}
%			\label{fig:discounted_return_vs_episode}
%			\vspace{-7em}
%	\end{minipage}}
%	\bigskip
%	\vspace{-7em}
%	\subfloat[Discounted return for specific tasks]{
%		\begin{minipage}[c][1\width]{0.5\textwidth}
%			\centering
%			\includegraphics[width=\textwidth,clip,trim={25 10 50 30}]{img/custom_lifelong.pdf}
%			\label{fig:custom_fig}
%			\vspace{-7em}
%	\end{minipage}}
%	%	\hfill
%	\subfloat[Algorithmic properties vs. $\prior$]{
%		\begin{minipage}[c][1\width]{0.5\textwidth}
%			\centering
%			\includegraphics[width=\textwidth,clip,trim={0 0 30 40}]{img/bounds_comparison.pdf}
%			\label{fig:bounds_comparison}
%			\vspace{-7em}
%	\end{minipage}}
%	\caption{
%		Experimental results
%		%		Average discounted return vs episodes (a) and tasks (b). 
%		%		(c) Comparison between the Lipschitz bound and the \rmax{} bound.
%		%		(d) Use of the prior $\prior$ vs model update.
%	}
%	\vspace{-1em}
%	\label{fig:experimental-results}
%\end{figure}\\
\begin{figure*}[t]%
	\centering%
	\begin{subfigure}{.49\textwidth}%
		\centering%
		\includegraphics[width=\textwidth,clip,trim={0 0 45 35}]{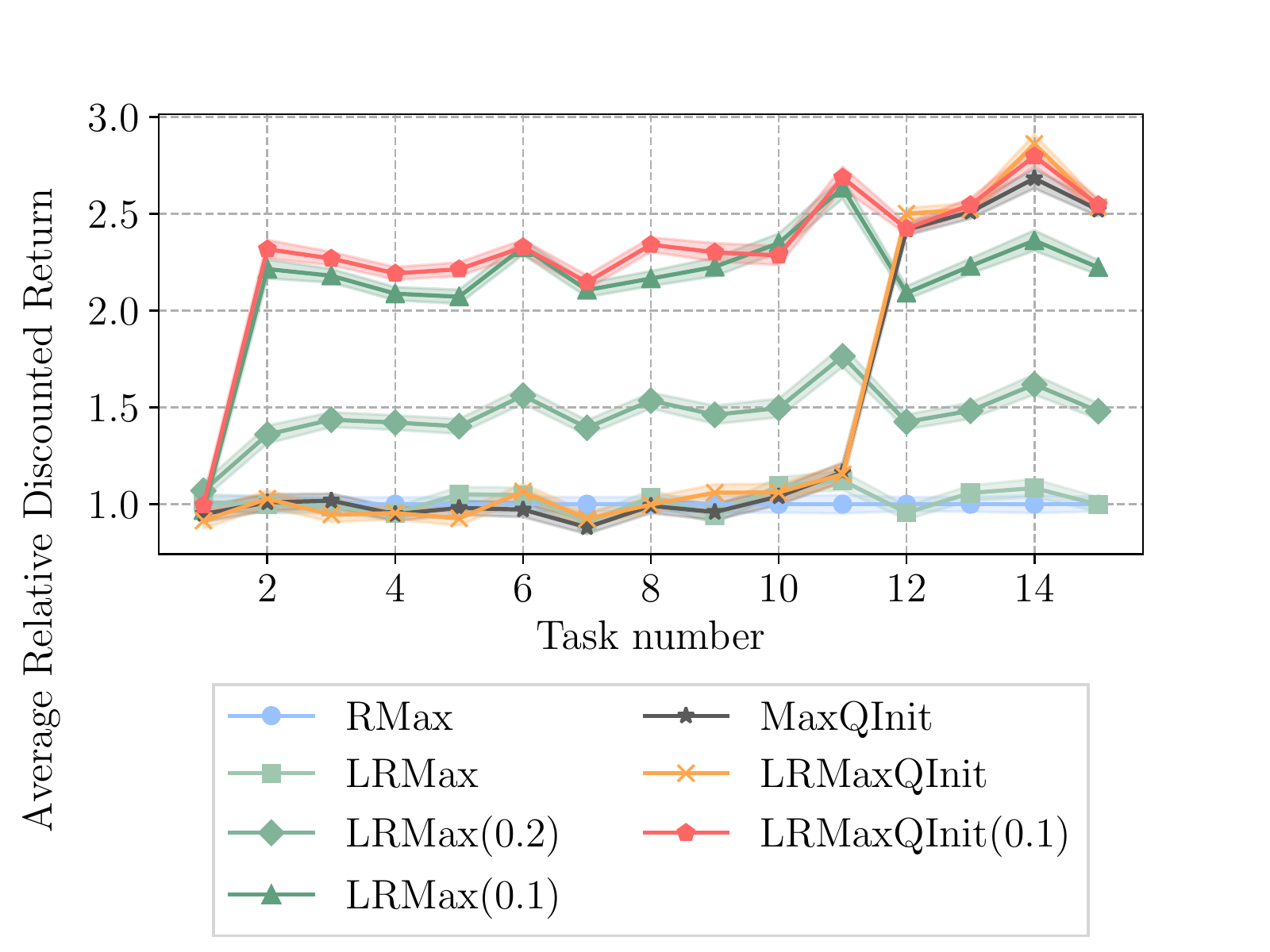}%
		\caption{Average discounted return vs. tasks}%
		\label{fig:discounted_return_vs_task}%
	\end{subfigure}
	\hfill
	\begin{subfigure}{.49\textwidth}%
		\centering%
		\includegraphics[width=\textwidth, clip, trim={10 0 40 40}]{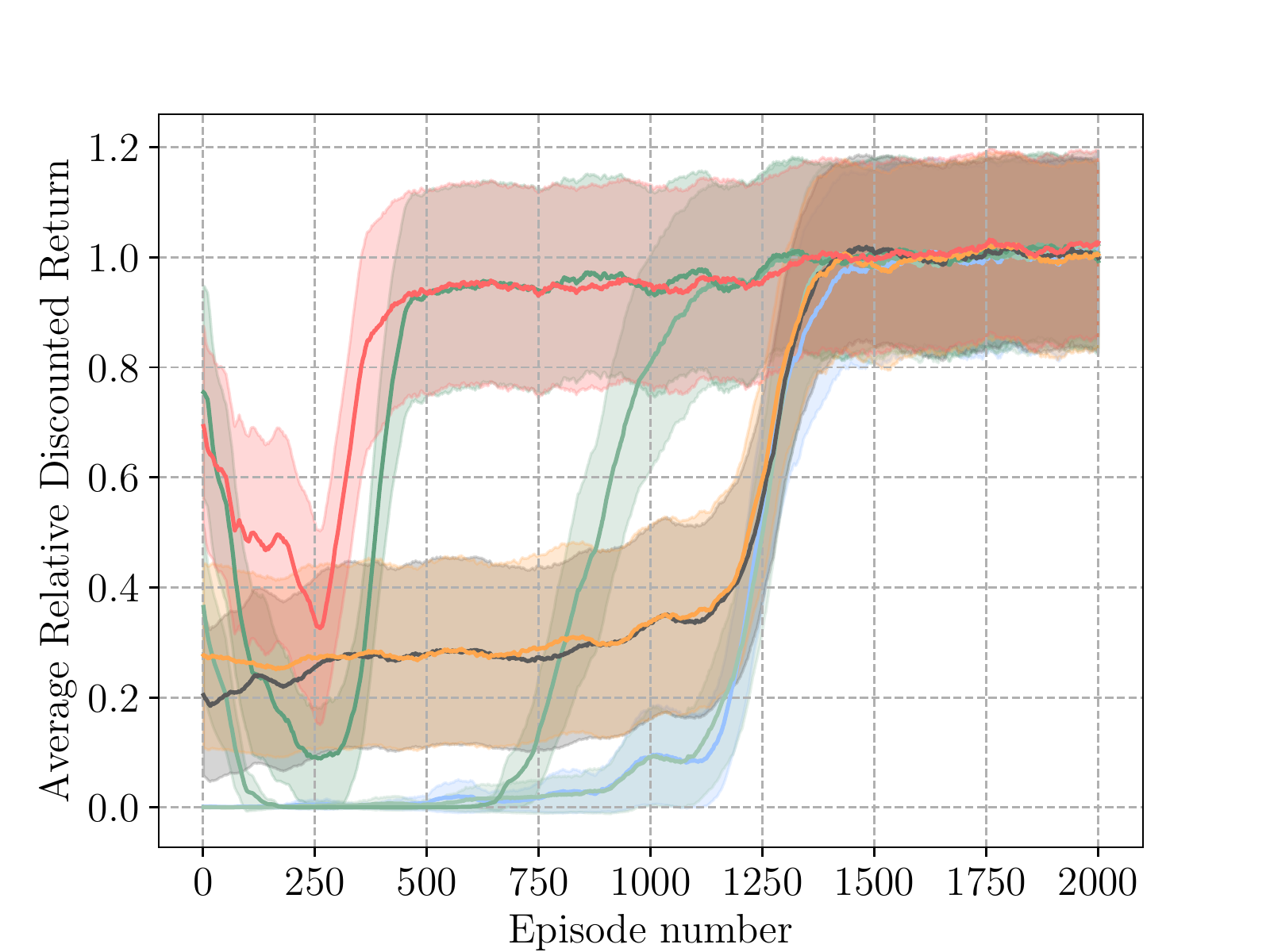}%
		\caption{Average discounted return vs. episodes}%
		\label{fig:discounted_return_vs_episode}%
	\end{subfigure}
	%\smallskip
	%\bigskip
	\begin{subfigure}{.49\textwidth}%
		\centering%
		\includegraphics[width=\textwidth, clip, trim={20 15 50 30}]{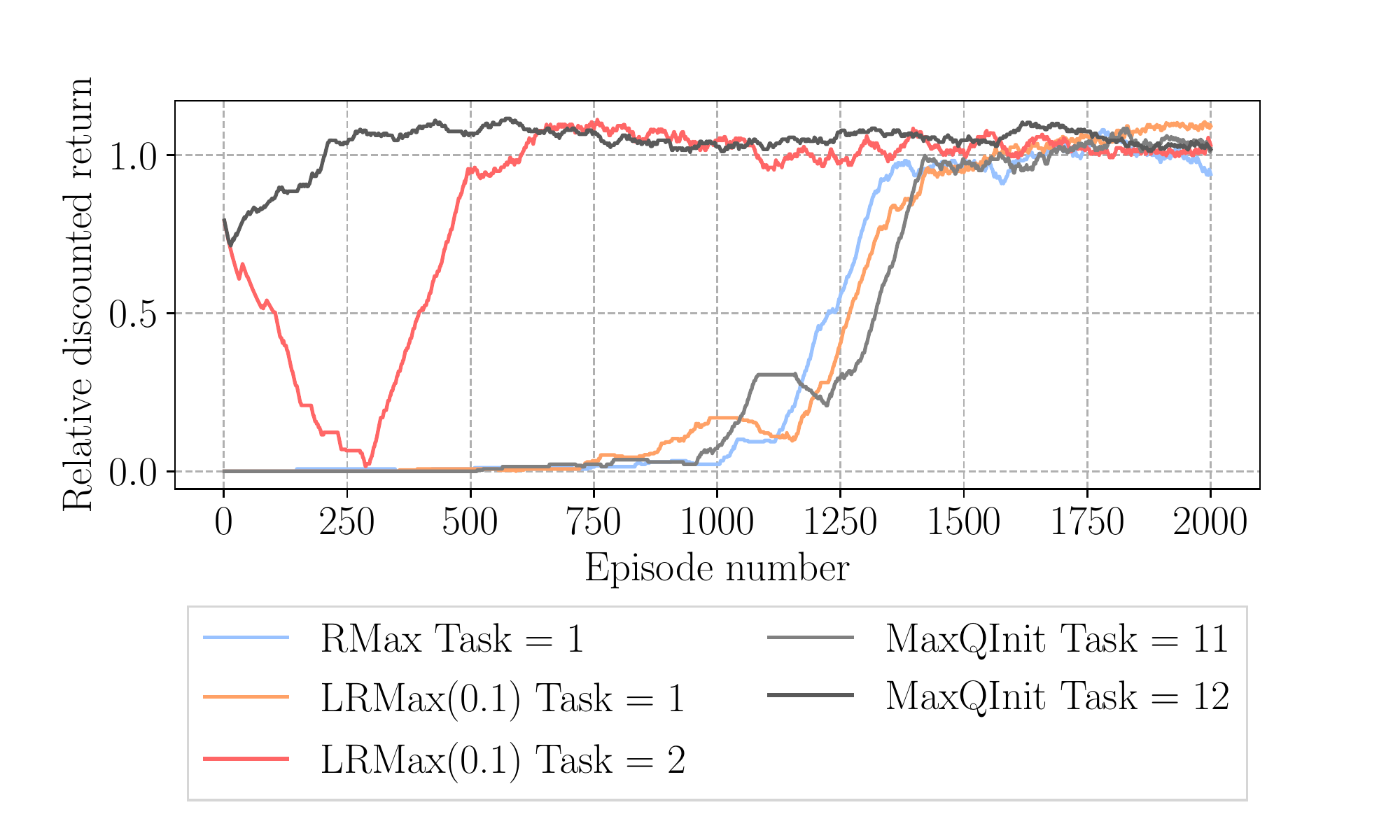}%
		\caption{Discounted return for specific tasks}%
		\label{fig:custom_fig}%
	\end{subfigure}
	\hfill
	\begin{subfigure}{.49\textwidth}%
		\centering%
		\includegraphics[width=\textwidth, clip, trim={0 15 30 42}]{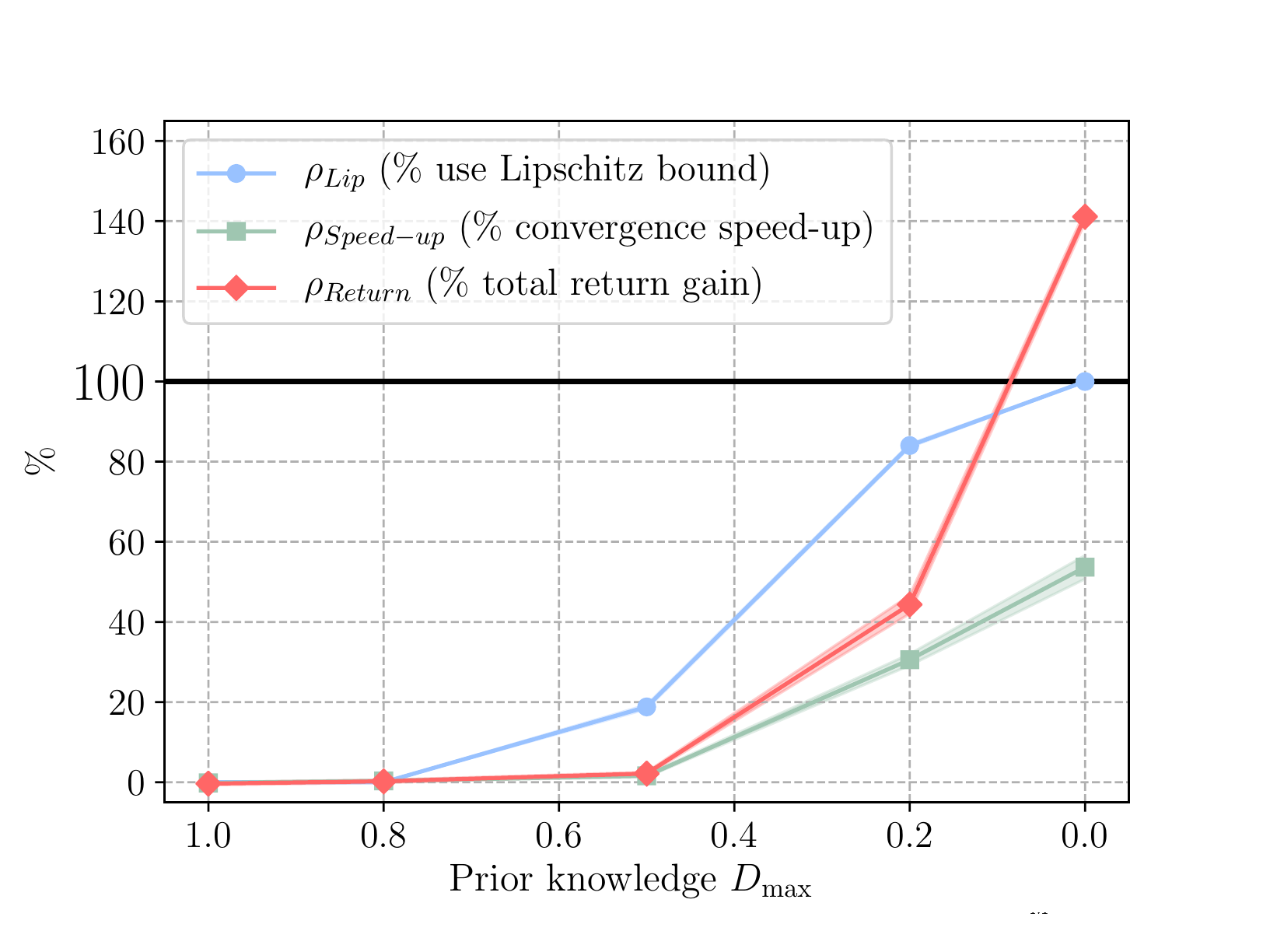}%
		\caption{Algorithmic properties vs. $\prior$}%
		\label{fig:bounds_comparison}%
	\end{subfigure}
	\caption{Experimental results. \lrmax{} benefits from an enhanced sample complexity thanks to the value-transfer method.}%
	\label{fig:experimental-results}
\end{figure*}

%& Agents
We evaluate different variants of \lrmax{} in a \lrl{} experiment.
The \rmax{} algorithm will be used as a no-transfer baseline.
\lrmax{}($x$) denotes Algorithm~\ref{alg:lrmax} with prior $\prior = x$.
\maxqinit{} denotes the \textsc{MaxQInit} algorithm from \citet{abel2018policy}, consisting in a state-of-the art PAC-MDP algorithm.
Both \lrmax{} and \maxqinit{} algorithms achieve value transfer by providing a tighter \ub{} on $Q^*$ than $\frac{1}{1 - \gamma}$.
Computing both \ub{}s and taking the minimum results in combining the two approaches.
We include such a combination in our study with the \lrmaxqinit{} algorithm.
Similarly, the latter algorithm benefiting from prior knowledge $\prior = x$ is denoted by \lrmaxqinit{}($x$).
For the sake of comparison, we only compare algorithms with the same features, namely, tabular, online, PAC-MDP methods, presenting non-negative transfer.

%& Envt
The environment used in all experiments is a variant of the ``tight'' task used by \citet{abel2018policy}.
It is an $11 \times 11$ grid-world, the initial state is in the centre, actions are the cardinal moves (Appendix, Section~\ref{sec:app:tight}).
%$\{ \uparrow, \rightarrow, \downarrow, \leftarrow \}$.
The reward is always zero except for the three goal cells in the upper-right corner.
Each sampled task has its own reward values, drawn from $[0.8, 1]$ for each of the three goal cells and its own probability of slipping (performing a different action than the one selected), picked in $[0, 0.1]$.
%Each time a task is sampled, a new reward value is drawn from $[0.8, 1]$ for each of the three goal cells and a probability of slipping (performing a different action than the one selected) is picked in $[0, 0.1]$.
Hence, tasks have different reward and transition functions.
Notice the distinction in applicability between \maxqinit{}, that requires the set of MDPs to be finite, and \lrmax{}, that can be used with any set of MDPs.
For the comparison between both to be possible, we drew tasks from a finite set of 5 MDPs.
We sample 15 tasks sequentially among this set, each run for 2000 episodes of length 10.
The operation is repeated 10 times to narrow the confidence intervals.
We set $n_{known} = 10$, $\delta=0.05$, and $\epsilon = 0.01$ (discussion in Appendix, Section~\ref{sec:app:nknown}).
Other \lrl{} experiments are reported in Appendix, Section~\ref{sec:app:additional-experiments}.

The results are reported in Figure~\ref{fig:experimental-results}.
Figure~\ref{fig:discounted_return_vs_task} displays the discounted return for each task, averaged across episodes.
Similarly, Figure~\ref{fig:discounted_return_vs_episode} displays the discounted return for each episode, averaged across tasks (same color code as Figure~\ref{fig:discounted_return_vs_task}).
Figure~\ref{fig:custom_fig} displays the discounted return for five specific instances, detailed below.
To avoid inter-task disparities, all the aforementioned discounted returns are displayed relative to an estimator of the optimal expected return for each task. 
For readability, Figures~\ref{fig:discounted_return_vs_episode} and~\ref{fig:custom_fig} display a moving average over 100 episodes.
Figure~\ref{fig:bounds_comparison} reports the benefits of various values of $\prior$ on the algorithmic properties.

In Figure~\ref{fig:discounted_return_vs_task}, we first observe that \lrmax{} benefits from the transfer method, as the average discounted return increases as more tasks are experienced.
Moreover, this advantage appears as early as the second task.
In contrast, \maxqinit{} requires to wait for task 12 before benefiting from transfer.
As suggested in Section~\ref{sec:improving-lrmax}, increasing amounts of prior knowledge allow the \lrmax{} transfer method to be more efficient: a smaller known \ub{} $\prior$ on $\modivhat{s}{a}{M}{\bar{M}}$ accelerates convergence.
Combining both approaches in the \lrmaxqinit{} algorithm outperforms all other methods.
Episode-wise, we observe in Figure~\ref{fig:discounted_return_vs_episode} that the \lrmax{} transfer method allows for faster convergence, \ie{}, lower sample complexity.
Interestingly, \lrmax{} exhibits three stages in the learning process.
1) The first episodes are characterized by a direct exploitation of the transferred knowledge, causing these episodes to yield high payoff.
This behavior is a consequence of the combined facts that the Lipschitz bound (Equation~\ref{eq:ub-on-lipschitz-bound}) is larger on promising regions of $\SA$ seen on previous tasks and the fact that \lrmax{} acts greedily \wrt{} that bound.
2) This high performance regime is followed by the exploration of unknown regions of $\SA$, in our case yielding low returns.
Indeed, as promising regions are explored first, the bound becomes tighter for the corresponding state-action pairs, enough for the Lipschitz bound of unknown pairs to become larger, thus driving the exploration towards low payoff regions.
Such regions are then identified and never revisited. % thereafter.
3) Eventually, \lrmax{} stops exploring and converges to the optimal policy.
Importantly, in all experiments, \lrmax{} never experiences negative transfer, as supported by the provability of the Lipschitz \ub{} with high probability.
\lrmax{} is at least as efficient as the no-transfer \rmax{} baseline.

Figure~\ref{fig:custom_fig} displays the collected returns of \rmax{}, \lrmax{}(0.1), and \maxqinit{} for specific tasks.
We observe that \lrmax{} benefits from transfer as early as Task 2, where the previous 3-stage behavior is visible.
\maxqinit{} takes until task 12 to leverage the transfer method.
%However, the bound it provides is tight enough to allow for almost zero exploration of the task.
However, the bound it provides is tight enough that it does not have to explore.

In Figure~\ref{fig:bounds_comparison}, we display the following quantities for various values of $\prior$:
$\rho_{Lip}$, the fraction of the time the Lipschitz bound was tighter than the \rmax{} bound $\frac{1}{1 - \gamma}$;
$\rho_{Speed\text{-}up}$, is the relative gain of time steps before convergence when comparing \lrmax{} to \rmax{}.
This quantity is estimated based on the last updates of the empirical model $\bar{M}$;
$\rho_{Return}$, is the relative total return gain on 2000 episodes of \lrmax{} \wrt{} \rmax{}.
First, we observe an increase of $\rho_{Lip}$ as $\prior$ becomes tighter.
This means that the Lipschitz bound of Equation~\ref{eq:ub-on-lipschitz-bound} becomes effectively smaller than $\frac{1}{1 - \gamma}$.
This phenomenon leads to faster convergence, indicated by $\rho_{Speed\text{-}up}$.
Eventually, this increased convergence rate allows for a net total return gain, as can be seen with the increase of $\rho_{Return}$.

% Conclusion:
Overall, in this analysis, we have showed that \lrmax{} benefits from an enhanced sample complexity thanks to the value-transfer method.
The knowledge of a prior $\prior$ increases this benefit.
The method is comparable to the \maxqinit{} method and has some advantages such as the early fitness for use and the applicability to infinite sets of tasks.
Moreover, the transfer is non-negative while preserving the PAC-MDP guarantees of the algorithm.
Additionally, we show in Appendix, Section~\ref{sec:app:prior-use-experiment} that, when provided with any prior $\prior$, \lrmax{} increasingly stops using it during exploration, confirming the claim of Section~\ref{sec:improving-lrmax} that providing $\prior$ enables transfer even if its value is of little importance.
%
%In addition to the analysis performed here, we show in Appendix Section~\ref{sec:app:prior-use-experiment} that, when provided with any prior $\prior$, \lrmax{} increasingly stops using this prior as the task is explored.
%This finding confirms the claim of Section~\ref{sec:improving-lrmax} that providing $\prior$ enables transfer even if its value is of little importance.

\section{Conclusion}

%& LC
We have studied theoretically the Lipschitz continuity property of the optimal Q-function in the MDP space \wrt{} a new metric. 
We proved a local Lipschitz continuity result, establishing that the optimal Q-functions of two close MDPs are themselves close to each other.
%This distance between Q-functions can be computed by Dynamic Programming.
%
%& Transfer method + LRMAX
We then proposed a value-transfer method using this continuity property with the Lipschitz \rmax{} algorithm, practically implementing this approach in the \lrl{} setting.
The algorithm preserves PAC-MDP guarantees, accelerates learning in subsequent tasks and exhibits no negative transfer.
Improvements of the algorithm were discussed in the form of prior knowledge on the maximum distance between models and online estimation of this distance.
As a non-negative, similarity-based, PAC-MDP transfer method, the \lrmax{} algorithm is the first method of the literature combining those three appealing features.
%
%& Expe
We showcased the algorithm in \lrl{} experiments and demonstrated empirically its ability to accelerate learning while not experiencing any negative transfer.
Notably, our approach can directly extend other PAC-MDP algorithms \citep{szita2010model,rao2012v,pazis2016improving,dann2017unifying} to the \l{} setting.
In hindsight, we believe this contribution provides a sound basis to non-negative value transfer via MDP similarity, a study that was lacking in the literature.
Key insights for the practitioner lie both in the theoretical analysis and in the practical derivation of a transfer scheme achieving non-negative transfer with PAC guarantees.
Further, designing scalable methods conveying the same intuition could be a promising research direction.

\section*{Acknowledgments}

\if\isanonymous1
Omitted for anonymity.
\else
We would like to thank Dennis Wilson for fruitful discussions and comments on the paper.
This research was supported by
the Occitanie region, France;
ISAE-SUPAERO (Institut Supérieur de l'Aéronautique et de l'Espace);
fondation ISAE-SUPAERO;
\'Ecole Doctorale Syst\`emes;
and ONERA, the French Aerospace Lab.
\fi

\bibliography{llrlbib}
%\bibliographystyle{icml2020}

%& Store counter value of propositions as propcounter in .aux file
%\storevalue{proposition}{propcounter}
%& Same for equations
%\newcounter{tempcounterbis}%
%\setcounter{tempcounterbis}{\value{equation} - 1}%
%\refstepcounter{tempcounterbis}\label{eqcounter}%

\if\doincludeappendixatbottom1
%\includepdf[pages=-]{llrl-appendix}

\newpage
%\maketitle
\onecolumn

\vspace*{1em}
\begin{center}
	{\LARGE\bf Lipschitz Lifelong Reinforcement Learning\\\vspace{0.5em}Appendix}
\end{center}
\vspace{1em}
\begin{tabular}[t]{p{\textwidth}}\centering \Large{\bf
	Erwan Lecarpentier\textsuperscript{\rm 1, 2}, David Abel\textsuperscript{\rm 3}, Kavosh Asadi\textsuperscript{\rm 3, 4}, Yuu Jinnai\textsuperscript{\rm 3},\\Emmanuel Rachelson\textsuperscript{\rm 1}, Michael L. Littman\textsuperscript{\rm 3}\\%
}%
	\vspace{3pt}  \normalsize  \affiliations_
\end{tabular}
\vspace{1em}

\section{Negative transfer}
\label{sec:example-negative-transfer}

%& What is negative transfer?
In the \lrl{} setting, it is reasonable to think that knowledge gained on previous MDPs could be re-used to improve the performance in new MDPs.
Such a practice, known as knowledge transfer, sometimes does cause the opposite effect, \ie{}, decreases the performance.
In such a case, we talk about \emph{negative transfer}.
Several attempt to formally define negative transfer have been done, but researchers hardly agree on a single definition, as \emph{performance} can be defined in various ways.
For instance, it can be characterized by the speed of convergence, the area under the learning curve, the final score of the learned policy or classifier, and many other things.
%Recently, an attempt to formalize negative transfer in supervised learning has been made by \citet{wang2019characterizing}.
Defining negative transfer is out of the scope of this paper, but let us give an example of why this phenomenon can be problematic.

%& An example of negative transfer
In their paper, \citet{song2016measuring} propose a transfer methods based on the metric between MDPs they introduce, stemming from the bi-simulation metric introduced by \citet{ferns2004metrics}.
In their method, a bi-simulation metric is computed between each pair of states belonging respectively to the source and target MDPs.
Roughly, this metric tells how \emph{different} are the transition and reward models corresponding to the states pairs, for the action maximizing th distance.
More precisely, if we note $(\transition, \Reward)$ and $(\bar{\transition}, \bar{\Reward})$ the models of two MDPs, and $\tuple{s}{s'} \in \S$ a state pair, the distance $d$ between $s$ and $s'$ is defined by
\begin{equation}
d(s, s') = \MAXENS{a \in \A}{\absnorm{\Rew{s}{a} - \Rewbar{s'}{a}} + c \; \wasserstein{1}{\tra{s}{a}{\cdot}}{\trabar{s'}{a}{\cdot}}}
\mthspc ,
\label{eq:bisim}
\end{equation}
where $c \in \R$ is a positive constant and $\wass{1}$ is the 1-Wasserstein metric.
For each state of the target model, the closest counterpart state (with the smallest bi-simulation distance) of the source MDP is identified and its learned \Qval{}s are used to initialize the \Qfun{} of the target MDP.
In their experiments, \citet{song2016measuring} run a standard Q-Learning algorithm~\citep{watkins1992q} with an $\epsilon$-greedy exploration strategy thereafter.

\begin{figure}[h]
	\centering
	\includegraphics[width=.8\textwidth]{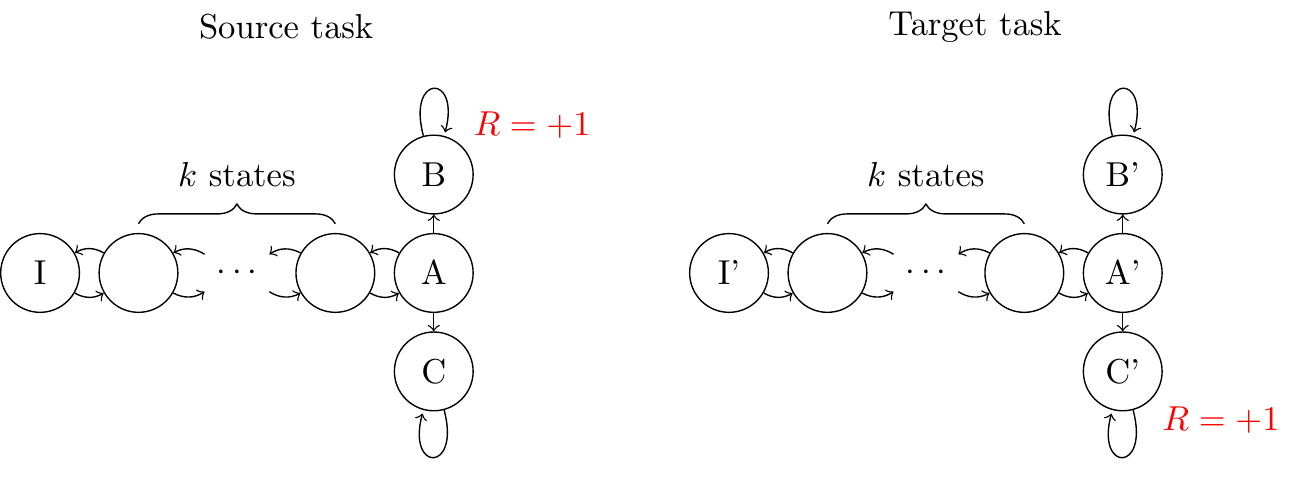}
	\caption{The T-shaped MDP transfer task.}
	\label{fig:negative-transfer-mdp}
\end{figure}

%& T-shaped MDP presentation
Let us now consider applying this method to a similar task to the T-shaped MDP transfer task proposed by \citet{taylor2009transfer}.
The source and target tasks are respectively described on the left and right sides of Figure~\ref{fig:negative-transfer-mdp}.
In each task, the states are represented by the circles and the arrows between them correspond to the available actions that allow to move from one state to the other.
The initial state of both tasks is the left state I for the source task and I$'$ for the target task.
Between the states I and A in the source task (respectively I$'$ to A$'$ in the target task) are $k$ states, $k$ being a parameter increasing the distance to travel from I to A (respectively I$'$ to A$'$).
The tasks are deterministic and the reward is zero everywhere except for the state B in the source task and C$'$ in the target task where a reward of $+1$ is received.
Consequently, the optimal policy in the source task is to go to the state A and then to the state B.
In the target task, the same applies except that a transition to state C should be applied in place of state B$'$ when the agent is in state A$'$.

%& Application of the method in the T-shaped MDP
Regardless of the parameters used in the bi-simulation metric of Equation~\ref{eq:bisim}, the direct state transfer method from \citet{song2016measuring} maps the following states together as they share the exact same model:
\begin{align*}
\text{I} \mthspc \longleftrightarrow & \mthspc \text{I}' \\
\text{k states} \mthspc \longleftrightarrow & \mthspc \text{k states} \\
\text{A} \mthspc \longleftrightarrow & \mthspc \text{A}'
\mthspc .
\end{align*}
Hence, during learning, the \Qfun{} of the target task is initialized with the values of the \Qfun{} of the source task.
Therefore, the behavior derived with the Q-Learning algorithm is the optimal policy of the source task, but in the target task.
Depending on the value of the learning rate of the algorithm, the time to favor action DOWN in state A$'$ instead of action UP can be arbitrarily long.
Also, depending on the value of $\epsilon$, the exploration of state C$'$ due to the $\epsilon$-greedy strategy can be arbitrarily unlikely.
Finally, the time needed for one of those two events to occur increases proportionally to the value of $k$, which can be arbitrarily large.

%& Lesson from this
This case illustrates the difficulty facing any transfer method in the general context of \lrl{}.
The method proposed by \citet{song2016measuring} can be highly efficient in some cases as they show in experiments, but the lack of theoretical guarantees makes negative transfer possible.
Generally, using a similarity measure such that the bi-simulation metric helps to discourage using some source tasks when the computed similarity is too low.
However, as we saw in the T-shaped MDP example, this rule is not absolute and the choice of the metric is important.
The approach we develop in this paper aims at avoiding negative transfer by providing a conservative transferred knowledge that is simply of no use when the similarity between source and target tasks is too low.
This is intuitive as we do not expect \emph{any} task to provide transferable knowledge to \emph{any} other task.

\section{Discussion on metrics and related notions}
\label{sec:app:local-mdps-distance-discussion}

A \emph{metric} on a set $X$ is a function $m: X \times X \rightarrow \R$ which has the following properties for any $x, y, z \in X$:
\begin{enumerate}[align=left]
	\item[P1.] $m(x, y) \geq 0$ (positivity),
	\item[P2.] $m(x, y) = 0 \Leftrightarrow x = y$ (positive definiteness),
	\item[P3.] $m(x, y) = m(y, x)$ (symmetry),
	\item[P4.] $m(x, z) \leq m(x, y) + m(y, z)$ (triangle inequality).
\end{enumerate}
If property P2 is not verified by $m$, but instead we have for any $x \in X$ that $m(x, x) = 0$, then $m$ is called a \emph{pseudo-metric}.
If $m$ only verifies P1, P2 and P4 then $m$ is called a \emph{quasi-metric}.
If $m$ only verifies P1 and P2 and if $X$ is a set of probability measures, then $m$ is called a \emph{divergence}.

From this, the pseudo-metric between models of Definition~\ref{def:pseudo-metric-between-models} is indeed a pseudo-metric as it is relative to a positive function $f$ that could be equal to zero and break property P2.

The local MDP dissimilarity between MDPs $\mdpdiv{s}{a}{M}{\bar{M}}$ of Proposition~\ref{proposition:local-lipschitz-continuity} does not respect properties P2 and P3, hence the name \emph{dissimilarity}.
The $\mdppm{s}{a}{M}{\bar{M}} = \min \left\{ \mdpdiv{s}{a}{M}{\bar{M}}, \mdpdiv{s}{a}{\bar{M}}{M} \right\}$ quantity, however, regains property P3 and is hence a pseudo-metric.
A noticeable consequence is that Proposition~\ref{proposition:local-lipschitz-continuity} is \qm{in the spirit} of a Lipschitz continuity result but cannot be called as such, hence the name \emph{pseudo-Lipschitz continuity}.

The same goes for the global dissimilarity $\mdpdiv{}{}{M}{\bar{M}} = \frac{1}{1 - \gamma} \max_{s, a \in \S \times \A} \left( \modiv{s}{a}{\gamma V^*_{\bar{M}}}{M}{\bar{M}} \right)$.
However, using $\min \left\{d_{M}^{\bar{M}}, d^{M}_{\bar{M}} \right\}$ allows to regain property 3 and makes this quantity a pseudo-metric again between MDPs.

\section{Proof of Proposition~\ref{proposition:local-lipschitz-continuity}}
\label{sec:proof-local-continuity}

\begin{notation}
	Given two sets $X$ and $Y$, we note $\functionspace{X}{Y}$ the set of functions defined on the domain $X$ with codomain $Y$.
	\label{notation:functions-set}
\end{notation}

\begin{lemma}
	\label{lemma:mdp-distance-fp}
	Given two MDPs $M, \bar{M} \in \M$, the following equation on $d \in \functionspace{\SA}{\R}$ is a \fp{} equation admitting a unique solution for any $\tuple{s}{a} \in \SA$:
	\begin{equation*}
	d_{s a} = \modiv{s}{a}{\gamma V^*_{\bar{M}}}{M}{\bar{M}} + \gamma \sum_{s' \in \S} T_{s s'}^a \max_{a' \in \A} d_{s' a'}.
	\end{equation*}
	We refer to this unique solution as $\mdpdiv{s}{a}{M}{\bar{M}}$.
\end{lemma}

\begin{proof}[Proof of Lemma~\ref{lemma:mdp-distance-fp}]
	The proof follows closely that in \cite{puterman2014markov} that proves that the Bellman operator over value functions is a contraction mapping.
	Let $L$ be the functional operator that maps any function $d \in \functionspace{\SA}{\R}$ to
	\begin{equation*}
	\FUNCTION{Ld}{\S \times \A}{\R}{s, a}{\modiv{s}{a}{\gamma V^*_{\bar{M}}}{M}{\bar{M}} + \gamma \sum_{s' \in \S} \tra{s}{a}{s'} \max_{a' \in \A} d_{s' a'} \mthspc .}
	\end{equation*}
	Then for $f$ and $g$, two functions from $\S \times \A$ to $\R$ and $\tuple{s}{a} \in \SA$, we have that
	\begin{align*}
	\label{eq:diff-func-operator-lemma}
	Lf_{s a} - Lg_{s a}
	& = \gamma \sum_{s' \in \S} \tra{s}{a}{s'} \left( \max_{a' \in \A} f_{s' a'} - \max_{a' \in \A} g_{s' a'} \right) \\
	& \leq \gamma \sum_{s' \in \S} \tra{s}{a}{s'} \max_{a' \in \A} \left(f_{s' a'} - g_{s' a'}\right) \\
	& \leq \gamma \inftynorm{f - g}
	\mthspc .
	\end{align*}
	Since this is true for any pair $\tuple{s}{a} \in \SA$, we have that
	\begin{equation*}
	\inftynorm{Lf - Lg} \leq \gamma \inftynorm{f - g}
	\mthspc .
	\end{equation*}
	Since $\gamma < 1$, $L$ is a contraction mapping in the metric space $\tuple{\functionspace{\SA}{\R}}{\inftynorm{\cdot}}$.
	This metric space being complete and non-empty, it follows by direct application of the Banach \fp{} theorem that the equation $d = L d$ admits a unique solution.
\end{proof}

\begin{proof}[Proof of Proposition~\ref{proposition:local-lipschitz-continuity}]
	The proof is by induction.
	The value iteration sequence of iterates $\left( \Q{n}{M} \right)_{n \in \N}$ of the optimal \Qfun{} of any MDP $M \in \M$ is defined for all $\tuple{s}{a} \in \SA$ by:
	\begin{align*}
	\Q{0}{M}(s, a) & = 0 \mthspc , \\
	\Q{n + 1}{M}(s, a) & = \Rew{s}{a} + \gamma \sum_{s' \in \S} \tra{s}{a}{s'} \max_{a' \in \A} \Q{n}{M}(s', a'), \mthspc \forall n \in \N \mthspc .
	\end{align*}
	Consider two MDPs $M, \bar{M} \in \M$.
	It is obvious that $\absnorm{\Q{0}{M}(s, a) - \Q{0}{\bar{M}}(s, a)} \leq \mdpdiv{s}{a}{M}{\bar{M}}$ for all $\tuple{s}{a} \in \SA$.
	Suppose the property $\absnorm{\Q{n}{M}(s, a) - \Q{n}{\bar{M}}(s, a)} \leq \mdpdiv{s}{a}{M}{\bar{M}}$ true at rank $n \in \N$ for all $\tuple{s}{a} \in \SA$.
	Consider now the rank $n + 1$ and a pair $\tuple{s}{a} \in \SA$:
	\begin{align*}
	\absnorm{\Q{n + 1}{M}(s, a) - \Q{n + 1}{\bar{M}}(s, a)}
	& = \absnorm{\Rew{s}{a}  - \Rewbar{s}{a} + \gamma \sum_{s' \in \S} \left[ \tra{s}{a}{s'} \max_{a' \in \A} \Q{n}{M}(s', a') - \trabar{s}{a}{s'} \max_{a' \in \A} \Q{n}{\bar{M}}(s', a') \right]} \\
	& \leq \absnorm{\Rew{s}{a} - \Rewbar{s}{a}} + \gamma \sum_{s' \in \S} \absnorm{\tra{s}{a}{s'} \max_{a' \in \A} \Q{n}{M}(s', a') - \trabar{s}{a}{s'} \max_{a' \in \A} \Q{n}{\bar{M}}(s', a')} \\
	& \leq \absnorm{\Rew{s}{a} - \Rewbar{s}{a}} + \gamma \sum_{s' \in \S} \max_{a' \in \A} \Q{n}{\bar{M}}(s', a') \absnorm{\tra{s}{a}{s'} - \trabar{s}{a}{s'}} \\ 
	& \phleq + \gamma \sum_{s' \in \S} \tra{s}{a}{s'} \absnorm{\max_{a' \in \A} \Q{n}{M}(s', a') - \max_{a' \in \A} \Q{n}{\bar{M}}(s', a')} \\
	& \leq \absnorm{\Rew{s}{a} - \Rewbar{s}{a}} + \sum_{s' \in \S} \gamma V^{*}_{\bar{M}}(s') \absnorm{\tra{s}{a}{s'} - \trabar{s}{a}{s'}} + \gamma \sum_{s' \in \S} \tra{s}{a}{s'} \max_{a' \in \A} \absnorm{\Q{n}{M}(s', a') - \Q{n}{\bar{M}}(s', a')} \\
	& \leq \modiv{s}{a}{\gamma \V{*}{\bar{M}}}{M}{\bar{M}} + \gamma \sum_{s' \in \S} \tra{s}{a}{s'} \max_{a'} \mdpdiv{s'}{a'}{M}{\bar{M}} \\
	& \leq \mdpdiv{s}{a}{M}{\bar{M}}
	\mthspc ,
	\end{align*}
	where we used Lemma~\ref{lemma:mdp-distance-fp} in the last inequality.
	Since $\Q{*}{M}$ and $\Q{*}{\bar{M}}$ are respectively the limits of the sequences $\left( \Q{n}{M} \right)_{n\in\N}$ and $\left( \Q{n}{\bar{M}} \right)_{n\in\N}$, it results from passage to the limit that
	\begin{equation*}
	\absnorm{\Q{*}{M}(s, a) - \Q{*}{\bar{M}}(s, a)} \leq \mdpdiv{s}{a}{M}{\bar{M}} \mthspc .
	\end{equation*}
	By symmetry, we also have $\absnorm{\Q{*}{M}(s, a) - \Q{*}{\bar{M}}(s, a)} \leq \mdpdiv{s}{a}{M}{\bar{M}}$ and we can take the minimum of the two valid \ub{}s, yielding:
	\begin{equation*}
	\absnorm{\Q{*}{M}(s, a) - \Q{*}{\bar{M}}(s, a)} \leq \min \left\{ \mdpdiv{s}{a}{M}{\bar{M}}, \mdpdiv{s}{a}{\bar{M}}{M} \right\} \mthspc ,
	\end{equation*}
	which concludes the proof.
\end{proof}

\section{Similar results to Proposition~\ref{proposition:local-lipschitz-continuity}}
\label{sec:app:other-results}

Similar results to Proposition~\ref{proposition:local-lipschitz-continuity} can be derived.
%& Global LC
First, an important consequence is the global pseudo-Lipschitz continuity result presented below.
\begin{proposition}[Global pseudo-Lipschitz continuity]
	\label{proposition:global-continuity}
	For two MDPs $M$, $\bar{M}$, for all $(s, a) \in \S \times \A$,
	\begin{equation}
	\label{eq:global-continuity}
	|Q^*_M(s, a) - Q^*_{\bar{M}}(s, a)| \leq \mdppm{}{}{M}{\bar{M}},
	\end{equation}
	with $\mdppm{}{}{M}{\bar{M}} \eqdef \min \SET{ \mdpdiv{}{}{M}{\bar{M}}, \mdpdiv{}{}{\bar{M}}{M} }$ and
	\begin{equation*}
	\mdpdiv{}{}{M}{\bar{M}} \triangleq \frac{1}{1 - \gamma} \MAXENS{s, a \in \S \times \A}{\modiv{s}{a}{\gamma V^*_{\bar{M}}}{M}{\bar{M}}}
	.
	\end{equation*}
\end{proposition}
From a pure transfer perspective, Equation~\ref{eq:global-continuity} is interesting since the right hand side does not depend on $s, a$.
Hence, the counterpart of the \ub{} of Equation~\ref{eq:local-ub}, namely,
\begin{equation*}
s, a \mapsto Q^*_{\bar{M}}(s, a) + \mdppm{}{}{M}{\bar{M}},
%\label{eq:global-ub}
\end{equation*}
is easier to compute.
Indeed, $\mdppm{}{}{M}{\bar{M}}$ can be computed once and for all, contrarily to $\mdppm{s}{a}{M}{\bar{M}}$ that needs to be evaluated for all $s, a$ pair.
However, we do not use this result for transfer because it is impractical to compute online.
Indeed, estimating the maximum in the definition of $\mdpdiv{}{}{M}{\bar{M}}$ can be as hard as solving both MDPs, which, when it happens, is too late for transfer to be useful.
\begin{proof}[Proof of Proposition~\ref{proposition:global-continuity}]
	The proof is by induction.
	We consider the sequence of value iteration iterates defined for any MDP $M \in \M$ for $\tuple{s}{a} \in \SA$ by
	\begin{align*}
	\Q{0}{M}(s, a) & = 0 \mthspc , \\
	\Q{n + 1}{M}(s, a) & = \Rew{s}{a} + \gamma \sum_{s' \in \S} \tra{s}{a}{s'} \max_{a' \in \A} \Q{n}{M}(s', a'), \mthspc \forall n \in \N \mthspc .
	\end{align*}
	Consider two MDPs $M, \bar{M} \in \M$.
	It is immediate for any $\tuple{s}{a} \in \SA$ that
	\begin{equation*}
	\absnorm{\Q{0}{M}(s, a) - \Q{0}{\bar{M}}(s, a)} \leq \mdpdiv{}{}{M}{\bar{M}}
	\mthspc ,
	\end{equation*}
	and, by symmetry, the result holds as well for $\mdpdiv{}{}{\bar{M}}{M}$.
	Suppose that it is true at rank $n \in \N$.
	Consider rank $n + 1$ and $\tuple{s}{a} \in \SA$, we have that:
	\begin{align*}
	\absnorm{\Q{n + 1}{M}(s, a) - \Q{n + 1}{\bar{M}}(s, a)}
	& \leq \modiv{s}{a}{\gamma \V{*}{\bar{M}}}{M}{\bar{M}} + \gamma \sum_{s' \in \S} \tra{s}{a}{s'} \max_{a' \in \A} \absnorm{\Q{n}{M}(s', a') - \Q{n}{\bar{M}}(s', a')} \\
	& \leq  \max_{\tuple{s}{a} \in \SA} \modiv{s}{a}{\gamma \V{*}{\bar{M}}}{M}{\bar{M}} + \gamma \sum_{s' \in \S} \tra{s}{a}{s'} \frac{1}{1 - \gamma} \max_{\tuple{s}{a} \in \SA} \modiv{s}{a}{\gamma \V{*}{\bar{M}}}{M}{\bar{M}} \\
	& \leq \max_{\tuple{s}{a} \in \SA} \modiv{s}{a}{\gamma \V{*}{\bar{M}}}{M}{\bar{M}} \left( 1 + \frac{\gamma}{1 - \gamma} \right) \\
	& \leq \mdpdiv{}{}{M}{\bar{M}}
	\mthspc .
	\end{align*}
	By symmetry, the results holds as well for $\mdpdiv{}{}{\bar{M}}{M}$ which concludes the proof by induction.
\end{proof}

%& Result 1
The second result is for the value function and is stated below.
\begin{proposition}[Local pseudo-Lipschitz continuity of the optimal value function]
	\label{proposition:local-lipschitz-continuity-v}
	For any two MDPs $M, \bar{M} \in \M$, for all $s \in \S$,
	\begin{equation*}
	\absnorm{\V{*}{M}(s) - \V{*}{\bar{M}}(s)} \leq \max_{a \in \A} \mdppm{s}{a}{M}{\bar{M}}
	\end{equation*}
	where the local MDP pseudo-metric $\mdppm{s}{a}{M}{\bar{M}}$ has the same definition as in Proposition~\ref{proposition:local-lipschitz-continuity}.
\end{proposition}

\begin{proof}[Proof of Proposition~\ref{proposition:local-lipschitz-continuity-v}]
	The proof follows exactly the same steps as the proof of Proposition~\ref{proposition:local-lipschitz-continuity}, \ie{}, by first constructing the value iteration sequence of iterates of the optimal value function, showing the result by induction for rank $n \in \N$ and then concluding with a passage to the limit.
\end{proof}

%& Result 2
Another result can be derived for the value of any policy $\pi$.
For the sake of generality, we state the result for any stochastic policy mapping states to distributions over actions.
Note that a deterministic policy is a stochastic policy mapping states to Dirac distributions over actions.
First, we state the result for the value function in Proposition~\ref{prop:local-lipischitz-for-any-policy-v} and then for the Q function in Proposition~\ref{prop:local-lipischitz-for-any-policy-q}.

\begin{proposition}[Local pseudo-Lipschitz continuity of the value function of any policy]
	\label{prop:local-lipischitz-for-any-policy-v}
	For any two MDPs $M, \bar{M} \in \M$, for any stochastic stationary policy $\pi$, for all $s \in \S$,
	\begin{equation*}
	\absnorm{\V{\pi}{M}(s) - \V{\pi}{\bar{M}}(s)} \leq \mdppmpi{s}{}{\pi}{M}{\bar{M}}
	\end{equation*}
	where $\mdppmpi{s}{}{\pi}{M}{\bar{M}} \eqdef \min \SET{ \mdpdivpi{s}{}{\pi}{M}{\bar{M}}, \mdpdivpi{s}{}{\pi}{\bar{M}}{M} }$ and $\mdpdivpi{s}{}{\pi}{M}{\bar{M}}$ is defined as the \fp{} of the following \fp{} equation on $d \in \functionspace{\S}{\R}$:
	\begin{equation*}
	d_s = \sum_{a \in \A} \pi(a \mid s) \left( \modpm{s}{a}{\gamma \V{\pi}{\bar{M}}}{M}{\bar{M}} + \gamma \sum_{s' \in \S} \tra{s}{a}{s'} d_{s'} \right)
	\mthspc .
	\end{equation*}
\end{proposition}

Before proving the Proposition, we show that the fixed point equation admits a unique solution in the following Lemma.

\begin{lemma}
	\label{lemma:mdp-distance-fp-any-policy-v}
	Given two MDPs $M, \bar{M} \in \M$, any stochastic stationary policy $\pi$, the following equation on $d \in \functionspace{\S}{\R}$ is a fixed-point equation admitting a unique solution for any $s \in \S$:
	\begin{equation*}
	d_s = \sum_{a \in \A} \pi(a \mid s) \left( \modpm{s}{a}{\gamma \V{\pi}{\bar{M}}}{M}{\bar{M}} + \gamma \sum_{s' \in \S} \tra{s}{a}{s'} d_{s'} \right)
	\mthspc .
	\end{equation*}
	We refer to this unique solution as $\mdpdivpi{s}{}{\pi}{M}{\bar{M}}$.
\end{lemma}

\begin{proof}[Proof of Lemma~\ref{lemma:mdp-distance-fp-any-policy-v}]
	Let $L$ be the functional operator that maps any function $d \in \functionspace{\S}{\R}$ to
	\begin{equation*}
	\FUNCTION{Ld}{\S}{\R}{s}{\sum_{a \in \A} \pi(a \mid s) \left( \modpm{s}{a}{\gamma \V{\pi}{\bar{M}}}{M}{\bar{M}} + \gamma \sum_{s' \in \S} \tra{s}{a}{s'} d_{s'} \right)}
	\mthspc .
	\end{equation*}
	Then for $f$ and $g$, two functions from $\S$ to $\R$, we have that
	\begin{align*}
	Lf_s - Lg_s & = \gamma \sum_{a \in \A} \pi(a \mid s) \sum_{s' \in \S} \tra{s}{a}{s'} \left( f_{s'} - g_{s'} \right) \\
	& \leq \gamma \inftynorm{f - g}
	\mthspc .
	\end{align*}
	Hence we have that $\inftynorm{Lf - Lg} \leq \gamma \inftynorm{f - g}$.
	Since $\gamma < 1$, $L$ is a contraction mapping in the metric space $\tuple{\functionspace{\S}{\R}}{\inftynorm{\cdot}}$.
	This metric space being complete and non-empty, it follows by direct application of the Banach \fp{} theorem that the equation $d = L d$ admits a unique solution.
\end{proof}

\begin{proof}[Proof of Proposition~\ref{prop:local-lipischitz-for-any-policy-v}]
	Consider a stochastic stationary stationary policy $\pi$.
	The value iteration sequence of iterates $\left( \V{\pi, n}{M} \right)_{n \in \N}$ of the value function of any MDP $M \in \M$ is defined for all $s \in \S$ by:
	\begin{align*}
	\V{\pi, 0}{M}(s) & = 0 \mthspc , \\
	\V{\pi, n + 1}{M}(s) & = \sum_{a \in \A} \pi(a \mid s) \left( \Rew{s}{a} + \gamma \sum_{s' \in \S} \tra{s}{a}{s'} \V{\pi, n}{M}(s') \right)
	\end{align*}
	Consider two MDPs $M, \bar{M} \in \M$.
	It is obvious that $\absnorm{\V{\pi, 0}{M}(s) - \V{\pi, 0}{\bar{M}}(s)} \leq \mdpdivpi{s}{}{\pi}{M}{\bar{M}}$ for all $s \in \S$.
	Suppose the property $\absnorm{\V{\pi, n}{M}(s) - \V{\pi, n}{\bar{M}}(s)} \leq \mdpdivpi{s}{}{\pi}{M}{\bar{M}}$ true at rank $n \in \N$ for all $s \in \S$.
	Consider now the rank $n + 1$ and the state $s \in \S$:
	\begin{align*}
	\absnorm{\V{\pi, n + 1}{M}(s) - \V{\pi, n + 1}{\bar{M}}(s)}
	& \leq \sum_{a \in \A} \pi(a \mid s) \absnorm{\Rew{s}{a} - \Rewbar{s}{a} + \gamma \sum_{s' \in \S} \left( \tra{s}{a}{s'} \V{\pi, n}{M}(s') - \trabar{s}{a}{s'} \V{\pi, n}{\bar{M}}(s') \right)} \\
	& \leq \sum_{a \in \A} \pi(a \mid s) \Bigg( \absnorm{\Rew{s}{a} - \Rewbar{s}{a}} + \gamma \sum_{s' \in \S} \V{\pi, n}{\bar{M}}(s') \absnorm{\tra{s}{a}{s'} - \trabar{s}{a}{s'}} \\
	& \phleq + \gamma \sum_{s' \in \S} \tra{s}{a}{s'} \absnorm{\V{\pi, n}{M}(s') - \V{\pi, n}{\bar{M}}(s')} \Bigg) \\
	& \leq \sum_{a \in \A} \pi(a \mid s) \left( \modpm{s}{a}{\gamma \V{\pi}{\bar{M}}}{M}{\bar{M}} + \gamma \sum_{s' \in \S} \tra{s}{a}{s'} \mdpdivpi{s'}{}{\pi}{M}{\bar{M}} \right) \\
	& \leq \mdpdivpi{s}{}{\pi}{M}{\bar{M}}
	\mthspc ,
	\end{align*}
	where we used Lemma~\ref{lemma:mdp-distance-fp-any-policy-v} in the last inequality.
	Since $\V{\pi}{M}$ and $\V{\pi}{\bar{M}}$ are respectively the limits of the sequences $\left( \V{\pi, n}{M} \right)_{n\in\N}$ and $\left( \V{\pi, n}{\bar{M}} \right)_{n\in\N}$, it results from passage to the limit that
	\begin{equation*}
	\absnorm{\V{\pi}{M}(s) - \V{\pi}{\bar{M}}(s)} \leq \mdpdivpi{s}{}{\pi}{M}{\bar{M}} \mthspc .
	\end{equation*}
	By symmetry, we also have $\absnorm{\V{\pi}{M}(s) - \V{\pi}{\bar{M}}(s)} \leq \mdpdivpi{s}{}{\pi}{\bar{M}}{M}$ and we can take the minimum of the two valid \ub{}s, yielding:
	\begin{equation*}
	\absnorm{\V{\pi}{M}(s) - \V{\pi}{\bar{M}}(s)} \leq \min \SET{ \mdpdivpi{s}{}{\pi}{M}{\bar{M}}, \mdpdivpi{s}{}{\pi}{\bar{M}}{M} } \mthspc ,
	\end{equation*}
	which concludes the proof.
\end{proof}

%& Result 3
\begin{proposition}[Local pseudo-Lipschitz continuity of the \Qfun{} of any policy]
	\label{prop:local-lipischitz-for-any-policy-q}
	For any two MDPs $M, \bar{M} \in \M$, for any stochastic stationary policy $\pi$, for all $\tuple{s}{a} \in \SA$,
	\begin{equation*}
	\absnorm{\Q{\pi}{M}(s, a) - \Q{\pi}{\bar{M}}(s, a)} \leq \mdppmpi{s}{a}{\pi}{M}{\bar{M}}
	\end{equation*}
	where $\mdppmpi{s}{a}{\pi}{M}{\bar{M}} \eqdef \min \SET{ \mdpdivpi{s}{a}{\pi}{M}{\bar{M}}, \mdpdivpi{s}{a}{\pi}{\bar{M}}{M} }$ and $\mdpdivpi{s}{a}{\pi}{M}{\bar{M}}$ is defined as the \fp{} of the following \fp{} equation on $d \in \functionspace{\SA}{\R}$:
	\begin{equation*}
	d_{s a} = \modpm{s}{a}{\gamma \V{\pi}{\bar{M}}}{M}{\bar{M}} + \gamma \sum_{\tuple{s'}{a'} \in \SA} \tra{s}{a}{s'} \pi(a' \mid s') d_{s' a'}
	\mthspc .
	\end{equation*}
\end{proposition}

Before proving the Proposition, we show that the fixed point equation admits a unique solution in the following Lemma.

\begin{lemma}
	\label{lemma:mdp-distance-fp-any-policy-q}
	Given two MDPs $M, \bar{M} \in \M$, any stochastic stationary policy $\pi$, the following equation on $d \in \functionspace{\SA}{\R}$ is a \fp{} equation admitting a unique solution for any $\tuple{s}{a} \in \SA$:
	\begin{equation*}
	d_{s a} = \modpm{s}{a}{\gamma \V{\pi}{\bar{M}}}{M}{\bar{M}} + \gamma \sum_{\tuple{s'}{a'} \in \SA} \tra{s}{a}{s'} \pi(a' \mid s') d_{s' a'}
	\mthspc .
	\end{equation*}
	We refer to this unique solution as $\mdpdivpi{s}{a}{\pi}{M}{\bar{M}}$.
\end{lemma}

\begin{proof}[Proof of Lemma~\ref{lemma:mdp-distance-fp-any-policy-q}]
	Let $L$ be the functional operator that maps any function $d \in \functionspace{\SA}{\R}$ to
	\begin{equation*}
	\FUNCTION{Ld}{\SA}{\R}{\tuple{s}{a}}{\modpm{s}{a}{\gamma \V{\pi}{\bar{M}}}{M}{\bar{M}} + \gamma \sum_{\tuple{s'}{a'} \in \SA} \tra{s}{a}{s'} \pi(a' \mid s') d_{s', a'} \mthspc .}
	\end{equation*}
	Then for $f$ and $g$, two functions from $\SA$ to $\R$, we have for all $\tuple{s}{a} \in \SA$ that
	\begin{align*}
	Lf_{s a} - Lg_{s a}
	& = \gamma \sum_{\tuple{s'}{a'} \in \SA} \tra{s}{a}{s'} \pi(a' \mid s') \left( Lf_{s' a'} - Lg_{s' a'} \right) \\
	& \leq \gamma \inftynorm{f - g}
	\mthspc .
	\end{align*}
	Hence we have that $\inftynorm{Lf - Lg} \leq \gamma \inftynorm{f - g}$.
	Since $\gamma < 1$, $L$ is a contraction mapping in the metric space $\tuple{\functionspace{\SA}{\R}}{\inftynorm{\cdot}}$.
	This metric space being complete and non-empty, it follows by direct application of the Banach \fp{} theorem that the equation $d = L d$ admits a unique solution.
\end{proof}

\begin{proof}[Proof of Proposition~\ref{prop:local-lipischitz-for-any-policy-q}]
	Consider a stochastic stationary policy $\pi$.
	The value iteration sequence of iterates $\left( \Q{\pi, n}{M} \right)_{n \in \N}$ of the Q function for the policy $\pi$ and MDP $M \in \M$ is defined for all $\tuple{s}{a} \in \SA$ by:
	\begin{align*}
	\Q{\pi, 0}{M}(s, a) & = 0 \mthspc , \\
	\Q{\pi, n + 1}{M}(s, a) & = \Rew{s}{a} + \gamma \sum_{\tuple{s'}{a'} \in \SA} \tra{s}{a}{s'} \pi(a' \mid s') \Q{\pi, n}{M}(s', a')
	\end{align*}
	Consider two MDPs $M, \bar{M} \in \M$.
	It is obvious that $\absnorm{\Q{\pi, 0}{M}(s, a) - \Q{\pi, 0}{\bar{M}}(s, a)} \leq \mdpdivpi{s}{a}{\pi}{M}{\bar{M}}$ for all $\tuple{s}{a} \in \SA$.
	Suppose the property $\absnorm{\Q{\pi, n}{M}(s, a) - \Q{\pi, n}{\bar{M}}(s, a)} \leq \mdpdivpi{s}{a}{\pi}{M}{\bar{M}}$ true at rank $n \in \N$ for all $\tuple{s}{a} \in \SA$.
	Consider now the rank $n + 1$ and the state-action pair $\tuple{s}{a} \in \SA$:
	\begin{align*}
	\absnorm{\Q{\pi, n + 1}{M}(s, a) - \Q{\pi, n + 1}{\bar{M}}(s, a)}
	& \leq \absnorm{\Rew{s}{a} - \Rewbar{s}{a}} + \gamma \sum_{\tuple{s'}{a'} \in \SA} \pi(a' \mid s') \absnorm{\tra{s}{a}{s'} \Q{\pi, n}{M}(s', a') - \trabar{s}{a}{s'} \Q{\pi, n}{\bar{M}}(s', a')} \\
	& \leq \absnorm{\Rew{s}{a} - \Rewbar{s}{a}} + \gamma \sum_{\tuple{s'}{a'} \in \SA} \pi(a' \mid s') \Q{\pi, n}{\bar{M}}(s', a') \absnorm{\tra{s}{a}{s'} - \trabar{s}{a}{s'}} \\
	& \phleq + \gamma \sum_{\tuple{s'}{a'} \in \SA} \pi(a' \mid s') \tra{s}{a}{s'} \absnorm{\Q{\pi, n}{M}(s', a') - \Q{\pi, n}{\bar{M}}(s', a')} \\
	& \leq \absnorm{\Rew{s}{a} - \Rewbar{s}{a}} + \sum_{s' \in \S} \gamma \V{\pi}{\bar{M}}(s') \absnorm{\tra{s}{a}{s'} - \trabar{s}{a}{s'}} + \gamma \sum_{\tuple{s'}{a'} \in \SA} \pi(a' \mid s') \tra{s}{a}{s'} d^{M, \bar{M}}_{\pi}(s', a') \\
	& \leq \modpm{s}{a}{\gamma \V{\pi}{\bar{M}}}{M}{\bar{M}} + \gamma \sum_{\tuple{s'}{a'} \in \SA} \tra{s}{a}{s'} \pi(a' \mid s') \mdpdivpi{s'}{a'}{\pi}{M}{\bar{M}} \\
	& \leq \mdpdivpi{s}{a}{\pi}{M}{\bar{M}}
	\mthspc ,
	\end{align*}
	where we used Lemma~\ref{lemma:mdp-distance-fp-any-policy-q} in the last inequality.
	Since $\Q{\pi}{M}$ and $\Q{\pi}{\bar{M}}$ are respectively the limits of the sequences $\left( \Q{\pi, n}{M} \right)_{n\in\N}$ and $\left( \Q{\pi, n}{\bar{M}} \right)_{n\in\N}$, it results from passage to the limit that
	\begin{equation*}
	\absnorm{\Q{\pi}{M}(s, a) - \Q{\pi}{\bar{M}}(s, a)} \leq \mdpdivpi{s}{a}{\pi}{M}{\bar{M}} \mthspc .
	\end{equation*}
	By symmetry, we also have $\absnorm{\Q{\pi}{M}(s, a) - \Q{\pi}{\bar{M}}(s, a)} \leq \mdpdivpi{s}{a}{\pi}{\bar{M}}{M}$ and we can take the minimum of the two valid \ub{}s, yielding for all $\tuple{s}{a} \in \SA$:
	\begin{equation*}
	\absnorm{\Q{\pi}{M}(s, a) - \Q{\pi}{\bar{M}}(s, a)} \leq \min \SET{ \mdpdivpi{s}{a}{\pi}{M}{\bar{M}}, \mdpdivpi{s}{a}{\pi}{\bar{M}}{M} } \mthspc ,
	\end{equation*}
	which concludes the proof.
\end{proof}

\section{Proof of Proposition~\ref{proposition:total-ub}}

\begin{proof}[Proof of Proposition~\ref{proposition:total-ub}]
	The result is clear for all $\tuple{s}{a} \notin K$ since the Lipschitz bounds are provably greater than $\Q{*}{M}$.
	For $\tuple{s}{a} \in K$, the result is shown by induction.
	Let us consider the Dynamic Programming~\citep{bellman1957dynamic} sequences converging to $\Q{*}{M}$ and $U$ whose definitions follow for all $\tuple{s}{a} \in \SA$ and for $n \in \N$:
	\begin{align*}
	&
	\begin{cases}
	\Q{0}{M}(s, a) = 0 \\
	\Q{n + 1}{M}(s, a) = \Rew{s}{a} + \gamma \sum_{s' \in \S} \tra{s}{a}{s'} \max_{a' \in \A} \Q{n}{M}(s', a')
	\end{cases}
	\mthspc ,
	\\
	&
	\begin{cases}
	U^0(s, a) = 0 \\
	U^{n + 1}(s, a) = \Rew{s}{a} + \gamma \sum_{s' \in \S} \tra{s}{a}{s'} \max_{a' \in \A} U^{n}(s', a')
	\end{cases}
	\end{align*}
	Obviously, we have at rank $n = 0$ that $\Q{0}{M}(s, a) \leq U^0(s, a)$ for all $\tuple{s}{a} \in \SA$.
	Suppose the property true at rank $n \in \N$ and consider rank $n + 1$:
	\begin{align*}
	\Q{n + 1}{M}(s, a) - U^{n + 1}(s, a)
	& = \gamma \sum_{s' \in \S} \tra{s}{a}{s'} \left( \max_{a' \in \A} \Q{n}{M}(s', a') - \max_{a' \in \A} U^{n}(s', a') \right) \\
	& \leq \gamma \sum_{s' \in \S} \tra{s}{a}{s'} \max_{a' \in \A} \left( \Q{n}{M}(s', a') - U^{n}(s', a') \right)\\
	& \leq 0
	\mthspc.
	\end{align*}
	Which concludes the proof by induction.
	The result holds by passage to the limit since the considered Dynamic Programming sequences converge to the true functions.
\end{proof}

\section{Proof of Proposition~\ref{proposition:ub-model-pseudo-metric}}

\begin{proof}[Proof of Proposition~\ref{proposition:ub-model-pseudo-metric}]
	Consider two tasks $M = (\transition, \Reward)$ and $\bar{M} = (\bar{\transition}, \bar{\Reward})$, with  $K$  and $\bar{K}$ the respective sets of state-action pairs where their learned models $\hat{M} = (\hat{\transition}, \hat{\Reward})$ and $\hat{\bar{M}} = ( \hat{\bar{\transition}}, \hat{\bar{\Reward}} )$ are known with accuracy $\epsilon$ in $\mathcal{L}_1$-norm with probability at least $1 - \delta$, \ie{}, we have that,
	\begin{align}%
	\label{eq:proof:epsilon-accurate-model}%
	\Pr \left(
	\begin{array}{lll}
	\absnorm{\Rew{s}{a} - \Rewhat{s}{a}} & \leq & \epsilon, \quad \forall \tuple{s}{a} \in K \quad \text{and}\\
	\onenorm{\tra{s}{a}{s'} - \trahat{s}{a}{s'}} & \leq & \epsilon, \quad \forall \tuple{s}{a} \in K \quad \text{and} \\
	\absnorm{\Rewbar{s}{a} - \Rewhatbar{s}{a}} & \leq & \epsilon, \quad \forall \tuple{s}{a} \in \bar{K} \quad \text{and} \\
	\onenorm{\trabar{s}{a}{s'} - \trahatbar{s}{a}{s'}} & \leq & \epsilon, \quad \forall \tuple{s}{a} \in \bar{K}
	\end{array}
	% \right\}
	\right) & \leq 1-\delta.
	\end{align}
	%& Note on union bound
	Importantly, notice that the probabilistic event of Inequality~\ref{eq:proof:epsilon-accurate-model} is the intersection of at most $4 \nS \nA$ individual events of estimating either $\Rew{s}{a}$, $\tra{s}{a}{s'}$, $\Rewbar{s}{a}$ or $\trabar{s}{a}{s'}$ with precision $\epsilon$.
	Each one of those individual events is itself true with probability at least $1 - \delta'$, where $\delta' \in (0, 1]$ is a parameter.
	For \emph{all} the individual events to be true at the same time, \ie{} for Inequality~\ref{eq:proof:epsilon-accurate-model} to be verified, one must apply Boole's inequality and set $\delta' = \delta / (4 \nS \nA)$ to ensure a total probability --- \ie{}, probability of the intersection of all the individual events --- of at least $1 - \delta$.
	
	We demonstrate now the result for each one of the three cases \begin{enumerate}[label=(\roman*), align=left]
		\item $\tuple{s}{a} \in K \cap \bar{K}$, \label{item:kbark}
		\item $\tuple{s}{a} \in K \cap \bar{K}^c$ and \label{item:kbarkc}
		\item $\tuple{s}{a} \in K^c \cap \bar{K}^c$, \label{item:kcbarkc}
	\end{enumerate}
	the case $\tuple{s}{a} \in K^c \cap \bar{K}$ being the symmetric of case \ref{item:kbarkc}.
	
	%& 1
	\ref{item:kbark} If $\tuple{s}{a} \in K \cap \bar{K}$, then we have $\epsilon$-close estimates of both models with high probability, as described by Inequality~\ref{eq:proof:epsilon-accurate-model}.
	By definition:
	\begin{equation}
	\modpm{s}{a}{\gamma \V{*}{\bar{M}}}{M}{\bar{M}} = \absnorm{\Rew{s}{a} - \Rewbar{s}{a}} + \gamma \sum_{s' \in \S} \V{*}{\bar{M}} (s') \absnorm{\tra{s}{a}{s'} - \trabar{s}{a}{s'}}
	\mthspc .
	\label{eq:model-distance-def}
	\end{equation}
	The first term of the right hand side of Equation~\ref{eq:model-distance-def} respects the following sequence of inequalities with probability at least $1 - \delta$:
	\begin{align}
	\absnorm{\Rew{s}{a} - \Rewbar{s}{a}}
	& \leq \absnorm{\Rew{s}{a} - \Rewhat{s}{a}} + \absnorm{\Rewhat{s}{a} - \Rewhatbar{s}{a}} + \absnorm{\Rewbar{s}{a} - \Rewhatbar{s}{a}} \nonumber \\
	& \leq \absnorm{\Rewhat{s}{a} - \Rewhatbar{s}{a}} + 2 \epsilon
	\mthspc .
	\label{eq:reward-gap-ub-1}
	\end{align}
	The second term of the right hand side of Equation~\ref{eq:model-distance-def} respects the following sequence of inequalities with probability at least $1 - \delta$:
	\begin{align}
	\gamma \sum_{s' \in \S} \V{*}{\bar{M}} (s') \absnorm{\tra{s}{a}{s'} - \trabar{s}{a}{s'}}
	& \leq \gamma \sum_{s' \in \S} \bar{V}(s') \left( \absnorm{\tra{s}{a}{s'} - \trahat{s}{a}{s'}} + \absnorm{\trahat{s}{a}{s'} - \hat{\bar{T}}_{s s'}^a} + \absnorm{\trabar{s}{a}{s'} - \hat{\bar{T}}_{s s'}^a} \right) \nonumber\\
	& \leq \gamma \max_{s'} \bar{V}(s') \sum_{s' \in \S} \absnorm{\tra{s}{a}{s'} - \trahat{s}{a}{s'}} + \gamma \sum_{s' \in \S} \bar{V}(s') \absnorm{\trahat{s}{a}{s'} - \hat{\bar{T}}_{s s'}^a} + \nonumber\\
	& \phleq \gamma \max_{s'} \bar{V}(s') \sum_{s' \in \S} \absnorm{\trabar{s}{a}{s'} - \hat{\bar{T}}_{s s'}^a} \nonumber\\
	& \leq \gamma \sum_{s' \in \S} \bar{V}(s') \absnorm{\trahat{s}{a}{s'} - \hat{\bar{T}}_{s s'}^a} + 2 \epsilon \gamma \max_{s' \in \S} \bar{V}(s').
	\label{eq:transition-gap-ub-1}
	\end{align}
	Replacing the Inequalities~\ref{eq:reward-gap-ub-1} and~\ref{eq:transition-gap-ub-1} in Equation~\ref{eq:model-distance-def} yields
	\begin{align*}
	\modiv{s}{a}{\gamma \V{*}{\bar{M}}}{M}{\bar{M}}
	& \leq \absnorm{\Rewhat{s}{a} - \Rewhatbar{s}{a}} + \gamma \sum_{s' \in \S} \bar{V}(s') \absnorm{\trahat{s}{a}{s'} - \hat{\bar{T}}_{s s'}^a} + 2 \epsilon + 2 \epsilon \gamma \max_{s' \in \S} \bar{V}(s') \\
	& \leq \modpm{s}{a}{\gamma \bar{V}}{\hat{M}}{\hat{\bar{M}}} + 2 \epsilon \left(1 + \gamma \max_{s' \in \S} \bar{V}(s') \right)
	\mthspc ,
	\end{align*}
	which holds with probability at least $1 - \delta$ and proves the Theorem for case \ref{item:kbark}.
	
	%& 2
	\ref{item:kbarkc} If $\tuple{s}{a} \in K \cap \bar{K}^c$, then we do not have an $\epsilon$-close estimate of $\trabar{s}{a}{\cdot}$ and $\Rewbar{s}{a}$.
	Similarly to the proof of case \ref{item:kbark}, we \ub{} sequentially the two terms of the right hand side of Equation~\ref{eq:model-distance-def}.
	With probability at least $1 - \delta$, we have the following:
	\begin{align}
	\absnorm{\Rew{s}{a} - \Rewbar{s}{a}} & \leq \absnorm{\Rew{s}{a} - \Rewhat{s}{a}} + \absnorm{\Rewhat{s}{a} - \Rewbar{s}{a}} \nonumber\\
	& \leq \epsilon + \max_{\bar{R} \in [0, 1]} \absnorm{\Rewhat{s}{a} - \bar{R}}.
	\label{eq:reward-gap-ub-2}
	\end{align}
	Similarly, with probability at least $1 - \delta$, we have:
	\begin{align}%
	\gamma \sum_{s' \in \S} \V{*}{\bar{M}} (s') \absnorm{\tra{s}{a}{s'} - \trabar{s}{a}{s'}} & \leq \gamma \sum_{s' \in \S} \bar{V}(s') \left( \absnorm{\tra{s}{a}{s'} - \trahat{s}{a}{s'}} + \absnorm{\trahat{s}{a}{s'} - \trabar{s}{a}{s'}} \right) \nonumber\\
	& \leq \gamma \max_{s' \in \S} \bar{V}(s') \epsilon + \gamma \max_{\bar{T} \in \setprobavect{\nS}} \sum_{s' \in \S} \bar{V}(s') \absnorm{\trahat{s}{a}{s'} - \bar{T}_{s'}}
	\mthspc ,%
	\label{eq:transition-gap-ub-2}%
	\end{align}%
	where $\setprobavect{\nS}$ is the set of probability vectors of size $\nS$.
	Combining inequalities~\ref{eq:reward-gap-ub-2} and~\ref{eq:transition-gap-ub-2}, we get the following with probability at least $1 - \delta$, by noticing $\modpm{s}{a}{\gamma \V{*}{\bar{M}}}{M}{\bar{M}}$ on the left hand side:
	\begin{align*}
	\modiv{s}{a}{\gamma \V{*}{\bar{M}}}{M}{\bar{M}} \leq \max_{\mdpvarbar \in \M} \modpm{s}{a}{\gamma \bar{V}}{\hat{M}}{\mdpvarbar} + \epsilon \left( 1 + \gamma \max_{s'} \bar{V}(s') \right),
	\end{align*}
	which is the expected result for case \ref{item:kbarkc}.
	
	%& 3
	\ref{item:kcbarkc} If $\tuple{s}{a} \in K^c \cap \bar{K}^c$, then we do not have $\epsilon$-close estimates of both tasks.
	In such a case, the result
	\begin{equation*}
	\modiv{s}{a}{\gamma \V{*}{\bar{M}}}{M}{\bar{M}} \leq \max_{\mdpvar, \mdpvarbar \in \M^2} \modpm{s}{a}{\gamma \bar{V}}{\mdpvar}{\mdpvarbar}
	\end{equation*}
	is straightforward by remarking that, as a consequence of Inequality~\ref{eq:proof:epsilon-accurate-model}, we have that $\V{*}{\bar{M}}(s) \leq \bar{V}(s)$ with probability at least $1 - \delta$.
\end{proof}

\section{Analytical calculation of $\modivhat{s}{a}{M}{\bar{M}}$ in Proposition~\ref{proposition:ub-model-pseudo-metric}}
\label{sec:app:analytical-calculation-dmodelhat}

Consider two tasks $M = (\transition, \Reward)$ and $\bar{M} = (\bar{\transition}, \bar{\Reward})$, with  $K$  and $\bar{K}$ the respective sets of state-action pairs where their learned models $\hat{M} = (\hat{\transition}, \hat{\Reward})$ and $\hat{\bar{M}} = ( \hat{\bar{\transition}}, \hat{\bar{\Reward}} )$ are known with accuracy $\epsilon$ in $\mathcal{L}_1$-norm with probability at least $1 - \delta$.
We note $\Vmax$, a known \ub{} on the maximum achievable value.
In the \wc{} where one does not have any information on the value of $\Vmax$, setting $\Vmax = \frac{1}{1 - \gamma}$ is a valid \ub{}.
We detail the computation of $\modivhat{s}{a}{M}{\bar{M}}$ for each cases: 1) $\tuple{s}{a} \in K \cap \bar{K}$, 2) $\tuple{s}{a} \in K \cap \bar{K}^c$, and 3) $\tuple{s}{a} \in K^c \cap \bar{K}^c$.
The case $\tuple{s}{a} \in K^c \cap \bar{K}$ being the symmetric of case 2), the same calculations apply.

1) If $\tuple{s}{a} \in K \cap \bar{K}$, we have
\begin{align*}
\modivhat{s}{a}{M}{\bar{M}}
& = \modpm{s}{a}{\gamma \bar{V}}{\hat{M}}{\hat{\bar{M}}} + 2 B \\
& = \absnorm{\Rewhat{s}{a} - \Rewhatbar{s}{a}} + \gamma \sum_{s' \in \S} \bar{V}(s') \absnorm{\trahat{s}{a}{s'} - \hat{\bar{T}}_{s s'}^a} + 2 \epsilon \left(1 + \gamma \max_{s' \in \S} \bar{V}(s') \right).
\end{align*}
Since $\tuple{s}{a}$ is a known state-action pair, everything is known and computable in this last equation.
Note that $\max_{s' \in \S} \bar{V}(s')$ can be tracked along the updates of $\bar{V}$ and thus its computation does not induce any additional computational complexity.

2) If $\tuple{s}{a} \in K \cap \bar{K}^c$, we have
\begin{align*}
\modivhat{s}{a}{M}{\bar{M}}
& = \max\limits_{\bar{\mu} \in \M} \modpm{s}{a}{\gamma \bar{V}}{\hat{M}}{\bar{\mu}} + B \\
& = \max_{\Rewbar{s}{a}, \trabar{s}{a}{s'}} \left( \absnorm{\Rewhat{s}{a} - \Rewbar{s}{a}} + \gamma \sum_{s' \in \S} \bar{V}(s') \absnorm{\trahat{s}{a}{s'} - \bar{T}_{ss'}^a} \right) + \epsilon \left(1 + \gamma \max_{s' \in \S} \bar{V}(s') \right), \\
& = \max_{r\in [0,1]} \absnorm{\Rewhat{s}{a} - r}
+ \gamma \max_{\substack{t \in [0, 1]^{\nS} \\ \text{\st{} } \sum_{s' \in \S} t_{s'} = 1}} \left( \sum_{s' \in \S} \bar{V}(s') \absnorm{\trahat{s}{a}{s'} - t_{s'}} \right)
+ \epsilon \left(1 + \gamma \max_{s' \in \S} \bar{V}(s') \right).
\end{align*}
First, we have
\begin{equation*}
\max_{r\in [0,1]} \absnorm{\Rewhat{s}{a} - r} = \max \left\{ \Rewhat{s}{a}, 1 - \Rewhat{s}{a} \right\}.
\end{equation*}
Maximizing over the variable $t \in [0, 1]^{\nS}$ such that $\sum_{s' \in \S} t_{s'} = 1$ is equivalent to maximizing a convex combination of the positive vector $\bar{V}$ whose terms are not independent as they must sum to one. 
This is easily solvable as a linear programming problem.
A straightforward (simplex-like) resolution procedure consists in progressively adding mass on the terms that will maximize the convex combination as follows:
\begin{itemize}
	\item $t_{s'} = 0, \mthspc \forall s'\in \S$
	\item $l$ = Sort states by decreasing values of $\bar{V}$
	\item While $\sum_{s \in \S} t_{s} < 1$
	\begin{itemize}
		\item $s'$ = pop first state in $l$
		\item Assign $t_{s'} \leftarrow \arg\max_{t \in [0,1]} \absnorm{\hat{T}_{ss'}^a - t}$ to $s'$ (note that $t_{s'}\in \{0,1\}$)
		\item If $\sum_{s \in \S} t_{s} > 1$, then $t_{s'} \leftarrow 1 - \sum_{s \in \S\setminus{s'}} t(s)$
	\end{itemize}
\end{itemize}
This allows calculating the maximum over transition models.

Notice that there is a simpler computation that almost always yields the same result (when it does not, it provides an \ub{}) and does not require the burden of the previous procedure. 
Consider the subset of states for which $\bar{V}(s')=\max_{s \in \S} \bar{V}(s)$ (often these are states in $\bar{K}^c$). 
Among those states, let us suppose there exists $s^+$, unreachable from $(s, a)$, according to $\hat{T}$, \ie{}, $\trahat{s}{a}{s^+} = 0$. 
If $\bar{M}$ has not been fully explored, as is often the case in \rmax{}, there may be many such states.
Then the distribution $t$ with all its mass on $s^+$ maximizes the $\max_{t \in [0,1]^{\nS}}$ term. 
Conversely, if such a state does not exist (that is, if for all such states $\trahat{s}{a}{s^+} > 0$), then $\max_{s \in \S} \bar{V}(s)$ is an \ub{} on the $\max_{t \in [0,1]^{\nS}}$ term.
Therefore:
\begin{equation*}
\max_{t\in [0, 1]^{\nS}} \left( \sum_{s' \in \S} \bar{V}(s') \absnorm{\trahat{s}{a}{s'} - t_{s'}} \right) \leq \max_{s \in \S} \bar{V}(s)
\mthspc ,
\end{equation*}
with equality in many cases.

3) If $\tuple{s}{a} \in K^c \cap \bar{K}^c$, the resolution is trivial and we have
\begin{align*}
\modivhat{s}{a}{M}{\bar{M}}
& = \max\limits_{\mu, \bar{\mu} \in \M^2} \modpm{s}{a}{\gamma \bar{V}}{\mu}{\bar{\mu}} \\
& = \max_{\Rew{s}{a}, T_{ss'}^a, \Rewbar{s}{a}, \bar{T}_{ss'}^a } \left( \absnorm{\Rew{s}{a} - \Rewbar{s}{a}} + \gamma \sum_{s' \in \S} \bar{V}(s') \absnorm{T_{ss'}^a - \bar{T}_{ss'}^a} \right) \\
& = \max_{r,\bar{r} \in [0,1]} \absnorm{r - \bar{r}} + \gamma \max_{\substack{t,\bar{t} \in [0,1]^{\nS} \\ \text{\st{} } \sum_{s \in \S} t_s = 1 \\ \text{and } \sum_{s \in \S} \bar{t}_s = 1 }} \sum_{s'\in \S} \bar{V}(s') \absnorm{t_{s'} - \bar{t}_{s'}} \\
& = 1 + 2 \gamma \max_{s \in \S} \bar{V}(s)
\mthspc .
\end{align*}
Overall, computing the value of the provided \ub{} in the three cases allows to compute $\modivhat{s}{a}{M}{\bar{M}}$ for all $\tuple{s}{a} \in \SA$.

\section{Proof of Proposition~\ref{proposition:asym-mdp-pseudo-distance-upperbound}}

\begin{lemma}
	\label{lemma:ub-mdp-distance-fp}
	Given two tasks $M, \bar{M} \in \M$, $K$ the set of state-action pairs for which $(\Reward, \transition)$ is known with accuracy $\epsilon$ in $\mathcal{L}_1$-norm with probability at least $1 - \delta$.
	If $\gamma (1 + \epsilon) < 1$, this equation on $\hat{d} \in \functionspace{\SA}{\R}$ is a \fp{} equation admitting a unique solution.
	\begin{equation*}
	\hat{d}_{s, a} = 
	\begin{cases}
	\modivhat{s}{a}{M}{\bar{M}} + \gamma \Big( \sum\limits_{s' \in \S} \hat{T}_{s s'}^a \max\limits_{a' \in \A} \hat{d}_{s', a'} + \epsilon \max\limits_{\tuple{s'}{a'} \in \SA} \hat{d}_{s', a'} \Big)
	\text{ if } \tuple{s}{a} \in K, \\
	\modivhat{s}{a}{M}{\bar{M}} + \gamma \max\limits_{\tuple{s'}{a'} \in \SA} \hat{d}_{s', a'}
	\text{ else}.
	\end{cases}
	\end{equation*}
	We refer to this unique solution as $\mdpdivhat{s}{a}{M}{\bar{M}}$.
\end{lemma}

\begin{proof}[Proof of Lemma~\ref{lemma:ub-mdp-distance-fp}]
	Let $L$ be the functional operator that maps any function $d \in \functionspace{\SA}{\R}$ to
	\begin{equation*}
	\FUNCTION{Ld}{\SA}{\R}{\tuple{s}{a}}{
		\begin{cases}
		\modivhat{s}{a}{M}{\bar{M}} + \gamma \left( \sum\limits_{s' \in \S} \trahat{s}{a}{s'} \max\limits_{a' \in \A} d_{s', a'} + \epsilon \max\limits_{\tuple{s'}{a'} \in \S \times \A} d_{s', a'} \right) \text{ if } \tuple{s}{a} \in K, \\
		\modivhat{s}{a}{M}{\bar{M}} + \gamma \max\limits_{\tuple{s'}{a'} \in \S \times \A} d_{s', a'}
		\text{ otherwise}.
		\end{cases}
	}
	\end{equation*}
	Let $f$ and $g$ be two functions from $\SA$ to $\R$.
	If $\tuple{s}{a} \in K$, we have
	\begin{align*}
	Lf_{s a} - Lg_{s a}
	& = \gamma \sum_{s' \in \S} \tra{s}{a}{s'} \left( \max_{a' \in \A} f_{s' a'} - \max_{a' \in \A} g_{s' a'} \right) + \gamma \epsilon \left( \max_{\tuple{s'}{a'} \in \SA} f_{s' a'} - \max_{\tuple{s'}{a'} \in \SA} g_{s' a'} \right) \\
	& \leq (\gamma + \gamma \epsilon) \left( \max_{\tuple{s'}{a'} \in \SA} f_{s' a'} - \max_{\tuple{s'}{a'} \in \SA} g_{s' a'} \right) \\
	& \leq \gamma (1 + \epsilon) \max_{\tuple{s'}{a'} \in \SA} \left( f_{s' a'} - g_{s' a'} \right) \\
	& \leq \gamma (1 + \epsilon) \inftynorm{f - g}
	\mthspc .
	\end{align*}
	If $\tuple{s}{a} \notin K$, we have
	\begin{align*}
	Lf_{s a} - Lg_{s a}
	& = \gamma \left( \max_{\tuple{s'}{a'} \in \SA} f_{s' a'} - \max_{\tuple{s'}{a'} \in \SA} g_{s' a'} \right) \\
	& \leq \gamma \max_{\tuple{s'}{a'} \in \SA} \left( f_{s' a'} - g_{s' a'} \right) \\
	& = \gamma (1 + \epsilon) \inftynorm{f - g}
	\mthspc .
	\end{align*}
	In both cases, $\inftynorm{Lf - Lg} \leq \gamma (1 + \epsilon) \inftynorm{f - g}$.
	If $\gamma (1 + \epsilon) < 1$, $L$ is a contraction mapping in the metric space $\tuple{\functionspace{\SA}{\R}}{\|\cdot\|_\infty}$.
	This metric space being complete and non-empty, it follows from Banach \fp{} theorem that $d = L d$ admits a single solution.
\end{proof}

\begin{proof}[Proof of Proposition~\ref{proposition:asym-mdp-pseudo-distance-upperbound}]
	Consider two MDPs $M, \bar{M} \in \M$.
	Before proving the result, notice that we shall put ourselves in the case of Proposition~\ref{proposition:ub-model-pseudo-metric}, for the \ub{} on the pseudometric between models $\modivhat{s}{a}{M}{\bar{M}}$ to be true \ub{}s with probability at least $1 - \delta$ for all $(s, a) \in \SA$.
	As seen in the proof of Proposition~\ref{proposition:ub-model-pseudo-metric}, this implies learning any reward or transition function with precision $\epsilon$ in $\L_1$-norm with probability at least $1 - \delta / (4 \nS \nA)$.
	
	The proof is done by induction, by calculating the values of $\mdpdiv{s}{a}{M}{\bar{M}}$ and $\mdpdivhat{s}{a}{M}{\bar{M}}$ following the value iteration algorithm.
	Those values can respectively be computed via the sequences of iterates $( d^{n} )_{n\in\N}$ and $( \hat{d}^{n} )_{n\in\N}$ defined as follows for all $\tuple{s}{a} \in \SA$:
	\begin{align*}
	\mdpdivn{s}{a}{0}{M}{\bar{M}} & = 0 \\
	\mdpdivn{s}{a}{n + 1}{M}{\bar{M}} & = \modiv{s}{a}{\gamma \V{*}{\bar{M}}}{M}{\bar{M}} + \gamma \sum_{s' \in \S} \tra{s}{a}{s'} \max_{a' \in \A} \mdpdivn{s'}{a'}{n}{M}{\bar{M}}
	\mthspc,
	\end{align*}
	and,
	\begin{align*}
	\mdpdivhatn{s}{a}{0}{M}{\bar{M}} & = 0, \\
	\mdpdivhatn{s}{a}{n + 1}{M}{\bar{M}} & = 
	\begin{cases}
	\modivhat{s}{a}{M}{\bar{M}} + \gamma \left( \sum\limits_{s' \in \S} \trahat{s}{a}{s'} \max\limits_{a' \in \A} \mdpdivhatn{s'}{a'}{n}{M}{\bar{M}} + \epsilon \max\limits_{\tuple{s'}{a'} \in \SA} \mdpdivhatn{s'}{a'}{n}{M}{\bar{M}} \right) \text{ if } \tuple{s}{a} \in K, \\
	\modivhat{s}{a}{M}{\bar{M}} + \gamma \max\limits_{\tuple{s'}{a'} \in \SA} \mdpdivhatn{s'}{a'}{n}{M}{\bar{M}}
	\text{ otherwise}.
	\end{cases}
	\end{align*}
	The proof at rank $n = 0$ is trivial.
	Let us assume the proposition $\mdpdivn{s}{a}{n}{M}{\bar{M}} \leq \mdpdivhatn{s}{a}{n}{M}{\bar{M}}, \mthspc \forall \tuple{s}{a} \in \SA$ true at rank $n \in \N$ and consider rank $n + 1$.
	There are two cases, depending on the fact that $\tuple{s}{a}$ is in $K$ or not.
	
	If $\tuple{s}{a} \in K$, we have
	\begin{align*}
	\mdpdivn{s}{a}{n + 1}{M}{\bar{M}} - \mdpdivhatn{s}{a}{n + 1}{M}{\bar{M}}
	& = \modiv{s}{a}{\gamma \V{*}{\bar{M}}}{M}{\bar{M}} - \modivhat{s}{a}{M}{\bar{M}} \\
	& \pheq + \gamma \sum_{s' \in \S} \left( \tra{s}{a}{s'} \max_{a' \in \A} \mdpdivn{s'}{a'}{n}{M}{\bar{M}} - \trahat{s}{a}{s'} \max\limits_{a' \in \A} \mdpdivhatn{s'}{a'}{n}{M}{\bar{M}} \right) \\
	& \pheq - \gamma \epsilon \max\limits_{\tuple{s'}{a'} \in \SA} \mdpdivhatn{s'}{a'}{n}{M}{\bar{M}}
	\mthspc .
	\end{align*}
	Using Proposition~\ref{proposition:ub-model-pseudo-metric}, we have that $\modivhat{s}{a}{M}{\bar{M}}$ is an \ub{} on $\modiv{s}{a}{\gamma \V{*}{\bar{M}}}{M}{\bar{M}}$ with probability at least $1 - \delta$.
	Hence
	\begin{equation*}
	\Pr \left( \modiv{s}{a}{\gamma \V{*}{\bar{M}}}{M}{\bar{M}} - \modivhat{s}{a}{M}{\bar{M}} \leq 0 \right) \geq 1 - \delta.
	\end{equation*}
	This plus the fact that $\mdpdivn{s}{a}{n}{M}{\bar{M}} \leq \mdpdivhatn{s}{a}{n}{M}{\bar{M}}, \mthspc \forall \tuple{s}{a} \in \SA$ by induction hypothesis, we have with probability at least $1 - \delta$,
	\begin{align*}
	\mdpdivn{s}{a}{n + 1}{M}{\bar{M}} - \mdpdivhatn{s}{a}{n + 1}{M}{\bar{M}}
	& \leq \gamma \sum_{s' \in \S} \max_{a' \in \A} \mdpdivhatn{s'}{a'}{n}{M}{\bar{M}} \left( \tra{s}{a}{s'} - \trahat{s}{a}{s'} \right) - \gamma \epsilon \max\limits_{\tuple{s'}{a'} \in \SA} \mdpdivhatn{s'}{a'}{n}{M}{\bar{M}} \\
	& \leq \gamma \max_{\tuple{s'}{a'} \in \SA} \mdpdivhatn{s'}{a'}{n}{M}{\bar{M}} \sum_{s' \in \S} \left( \tra{s}{a}{s'} - \trahat{s}{a}{s'} \right) - \gamma \epsilon \max_{\tuple{s'}{a'} \in \SA} \mdpdivhatn{s'}{a'}{n}{M}{\bar{M}} \\
	& \leq \gamma \max_{\tuple{s'}{a'} \in \SA} \mdpdivhatn{s'}{a'}{n}{M}{\bar{M}} \left( \onenorm{T - \hat{T}} - \epsilon \right)
	\mthspc .
	\end{align*}
	Since $\Pr \left( \onenorm{T - \hat{T}} \leq \epsilon \right) \geq 1 - \delta$, we have with probability at least $1 - \delta$,
	\begin{equation*}
	\mdpdivn{s}{a}{n + 1}{M}{\bar{M}} - \mdpdivhatn{s}{a}{n + 1}{M}{\bar{M}}
	\leq \gamma \max_{\tuple{s'}{a'} \in \SA} \mdpdivhatn{s'}{a'}{n}{M}{\bar{M}} \left( \epsilon - \epsilon \right) = 0,
	\end{equation*}
	which concludes the proof in the first case case.
	
	Conversely, if $\tuple{s}{a} \notin K$, we have
	\begin{align*}
	\mdpdivn{s}{a}{n + 1}{M}{\bar{M}} - \mdpdivhatn{s}{a}{n + 1}{M}{\bar{M}}
	& = \modiv{s}{a}{\gamma \V{*}{\bar{M}}}{M}{\bar{M}} - \modivhat{s}{a}{M}{\bar{M}} + \gamma \sum_{s' \in \S} \tra{s}{a}{s'} \max_{a' \in \A} \mdpdivn{s'}{a'}{n}{M}{\bar{M}} - \gamma \max\limits_{\tuple{s'}{a'} \in \SA} \mdpdivhatn{s'}{a'}{n}{M}{\bar{M}}
	\mthspc .
	\end{align*}
	Using the same reasoning than in case $\tuple{s}{a} \in K$, we have with probability higher than $1 - \delta$,
	\begin{align*}
	\mdpdivn{s}{a}{n + 1}{M}{\bar{M}} - \mdpdivhatn{s}{a}{n + 1}{M}{\bar{M}}
	& \leq \gamma \sum_{s' \in \S} \tra{s}{a}{s'} \max_{a' \in \A} \mdpdivhatn{s'}{a'}{n}{M}{\bar{M}} - \gamma \max_{\tuple{s'}{a'} \in \SA} \mdpdivhatn{s'}{a'}{n}{M}{\bar{M}} \\
	& \leq \gamma \max_{\tuple{s'}{a'} \in \SA} \mdpdivhatn{s'}{a'}{n}{M}{\bar{M}} - \gamma \max_{\tuple{s'}{a'} \in \SA} \mdpdivhatn{s'}{a'}{n}{M}{\bar{M}} \\
	& \leq 0,
	\end{align*}
	which concludes the proof in the second case.
\end{proof}

\section{Proof of Proposition~\ref{proposition:computational-complexity}}

\begin{proof}[Proof of Proposition~\ref{proposition:computational-complexity}]
	The cost of \lrmax{} is constant on most time steps since the action is greedily chosen \wrt{} the \ub{} on the optimal \Qfun{}, which is a lookup table.
	%& Number of VI
	Let $N \in \N$ be the number of source tasks that have been learned by \lrmax{} during a \lrl{} experiment.
	When updating a new state-action pair, \ie{}, labeling it as a known pair, the algorithm performs $2N$ Dynamic Programming (DP) computations to update the induced Lipschitz bounds (Equation~\ref{eq:asym-mdp-pseudo-distance-upperbound}) plus one DP computation to update the total-bound (Equation~\ref{eq:total-ub}).
	In total, we apply $(2N + 1)$ DP computations for each state-action pair update.
	As at most $\nS \nA$ state-action pairs are updated during the exploration of the current MDP, the total number of DP computations is at most $\nS \nA (2N + 1)$, for which we use the value iteration algorithm.
	
	%& Cost of 1 VI
	We use the value iteration as a Dynamic Programming method.
	\citet{strehl2009reinforcement} report the minimum number of iterations needed by the value iteration algorithm to estimate a value function (or \Qfun{} in our case) that is $\epsilon_Q$-close to the optimum in maximum norm.
	This minimum number is given by
	\begin{equation*}
	\ceil[\bigg]{\frac{1}{1 - \gamma} \ln \left( \frac{1}{\epsilon_Q (1 - \gamma)} \right)}
	\mthspc .
	\end{equation*}
	Each iteration has a cost $\nS^2 \nA$.
	Overall, the cost of all the DP computations in a complete run of \lrmax{} is
	\begin{equation*}
	\bigotilde \left( \frac{\nS^3 \nA^2 N}{1 - \gamma} \ln \left( \frac{1}{\epsilon_Q (1 - \gamma)} \right) \right)
	\mthspc .
	\end{equation*}
	This, plus the constant cost $\bigo (1)$ applied on each one of the $\ntimesteps$ decision epochs concludes the proof.
\end{proof}

\section{Proof of Proposition~\ref{proposition:max-distance-estimate}}

\begin{proof}[Proof of Proposition~\ref{proposition:max-distance-estimate}]
	Consider an algorithm producing $\epsilon$-accurate model estimates $\modivhat{s}{a}{M}{\bar{M}}$ for a subset $K$ of $\SA$ after interacting with any two MDPs $M, \bar{M} \in \M$.
	Assume $\modivhat{s}{a}{M}{\bar{M}}$ to be an \ub{} of $\modiv{s}{a}{\gamma V^*_{\bar{M}}}{M}{\bar{M}}$ for any $(s, a) \notin K$.
	These assumptions are guaranteed with high probability by Proposition~\ref{proposition:ub-model-pseudo-metric} while running Algorithm~\ref{alg:lrmax} in the \lrl{} setting.
	Then, for any $(s, a) \in \SA$ and any two MDPs $M, \bar{M} \in \M$, we have that
	\begin{align*}
	\begin{array}{lclr}
	\modivhat{s}{a}{M}{\bar{M}}
	& =
	& \modiv{s}{a}{\gamma V^*_{\bar{M}}}{M}{\bar{M}} \pm \epsilon & \text{ if } (s, a) \in K \\
	\modivhat{s}{a}{M}{\bar{M}}
	& \geq
	& \modiv{s}{a}{\gamma V^*_{\bar{M}}}{M}{\bar{M}}
	& \text{ else.}
	\end{array}
	\end{align*}
	Particularly, $\modivhat{s}{a}{M}{\bar{M}} + \epsilon \geq \modiv{s}{a}{\gamma V^*_{\bar{M}}}{M}{\bar{M}}$ for all $(s, a) \in \SA$ and any $M, \bar{M} \in \M$.
	By definition of $\prior(s, a)$, this implies that, for all $(s, a) \in \SA$,
	\begin{equation}
	\max_{M, \bar{M} \in \tilde{\M}} \modivhat{s}{a}{M}{\bar{M}} + \epsilon \geq \prior(s, a) \mthspc ,
	\label{eq:max-estimator-ub}
	\end{equation}
	where $\tilde{M}$ is the set of possible tasks in the considered \lrl{} experiment.
	Consider $\hat{\M}$, the set of sampled MDPs which allows to define $\hat{D}_{\max}(s, a) = \max_{M, \bar{M} \in \hat{\M}} \modivhat{s}{a}{M}{\bar{M}}$ as the maximum model distance for all the experienced MDPs at $(s, a) \in \SA$.
	We have that
	\begin{equation*}
	\hat{D}_{\max}(s, a) = \max_{M, \bar{M} \in \tilde{\M}} \modivhat{s}{a}{M}{\bar{M}} \mthspc ,
	\end{equation*}
	only if two MDPs maximizing the right hand side of this equation belong to $\hat{M}$.
	If it is the case, then Equation~\ref{eq:max-estimator-ub} imply that
	\begin{equation}
	\hat{D}_{\max}(s, a) + \epsilon \geq \prior(s, a) \mthspc .
	\label{eq:ub-dmax-result}
	\end{equation}
	Overall, we require the two MDPs maximizing $\max_{M, \bar{M} \in \tilde{\M}} \modivhat{s}{a}{M}{\bar{M}}$ to be sampled for Equation~\ref{eq:ub-dmax-result} to hold.
	Let us now derive the probability that those two MDPs have been sampled.
	We note them $M_1$ and $M_2$.
	There may exist more candidates for the maximization but, for the sake of generality, we put ourselves in the case where only two MDPs achieve the maximization.
	Let us consider drawing $m \in \N$ tasks.
	We note $p_1$ (respectively $p_2$) the probability of sampling $M_1$ (respectively $M_2$) each time a task is sampled.
	We note $X_1$ (respectively $X_2$) the random variable of the first occurrence of the task $M_1$ (respectively $M_2$) among the $m$ trials.
	Hence, the probability of sampling $M_1$ for the first time at trial $k \in \intrange{1}{m}$ is given by the geometric law and is equal to
	\begin{equation*}
	\Pr \left( X_1 = k \right) = p_1 \left( 1 - p_1 \right)^{k - 1} \mthspc .
	\end{equation*}
	Additionally, the probability of sampling $M_1$ at least once in the first $m$ trials is given by the cumulative distribution function:
	\begin{equation}
	\Pr \left( X_1 \leq m \right) = 1 - (1 - p_1)^m \mthspc .
	\label{eq:cdf-geo}
	\end{equation}
	We are interested in the probability of the event that $M_1$ \emph{and} $M_2$ have been sampled in the $m$ first trials, \ie{} $\Pr \left( X_1 \leq m \cap X_2 \leq m \right)$.
	Following the rule of addition for probabilities, we have that,
	\begin{equation*}
	\Pr \left( X_1 \leq m \cap X_2 \leq m \right) = \Pr \left( X_1 \leq m \right) + \Pr \left( X_2 \leq m \right) - \Pr \left( X_1 \leq m \cup X_2 \leq m \right) \mthspc .
	\end{equation*}
	Given that the event of sampling either $M_1$ or $M_2$ during a single trial happens with probability $p_1 + p_2$, we have by analogy with Equation~\ref{eq:cdf-geo} that $\Pr \left( X_1 \leq m \cup X_2 \leq m \right) = 1 - (1 - (p_1 + p_2))^m$.
	As a result, the following holds:
	\begin{align*}
	\Pr \left( X_1 \leq m \cap X_2 \leq m \right)
	& = 1 - (1 - p_1)^m + 1 - (1 - p_2)^m - \left( 1 - (1 - (p_1 + p_2))^m \right) \\
	& = 1 - (1 - p_1)^m - (1 - p_2)^m + (1 - (p_1 + p_2))^m \\
	& \geq 1 - 2 (1 - \pmin)^m + (1 - 2 \pmin)^m
	\mthspc .
	\end{align*}
	As said earlier, Equation~\ref{eq:ub-dmax-result} holds if $M_1$ and $M_2$ have been sampled during the first $m$ trials, which imply that the probability for Equation~\ref{eq:ub-dmax-result} to hold is at least equal to the probability of sampling both tasks.
	Formally,
	\begin{align*}
	\Pr \left( \hat{D}_{\max}(s, a) + \epsilon \geq \prior(s, a) \right) & \geq \Pr \left( X_1 \leq m \cap X_2 \leq m \right) \\
	& \geq 1 - 2 (1 - \pmin)^m + (1 - 2 \pmin)^m
	\mthspc .
	\end{align*}
	In turn, if $m$ verifies $2 (1 - \pmin)^m - (1 - 2\pmin)^m \leq \delta$, then $1 - 2 (1 - \pmin)^m + (1 - 2 \pmin)^m \geq 1 - \delta$ and $\Pr \left( \hat{D}_{\max}(s, a) + \epsilon \geq \prior(s, a) \right) \geq 1 - \delta$, which concludes the proof.
\end{proof}

\section{Discussion on an \ub{} on distances between MDP models}
\label{sec:app:dmax}

Section \ref{sec:improving-lrmax} introduced the idea of exploiting \emph{prior} knowledge on the maximum distance between two MDP models. 
This idea begs for a more detailed discussion.
%In full generality, the most distant MDP $\bar{M}$ from $M$ is at a distance $\modpm{s}{a}{\gamma V^*_{\bar{M}}}{M}{\bar{M}}=\frac{1+\gamma}{1-\gamma}$, given Equation \ref{eq:local-model-pseudo-metric}.
Consider two MDPs $M$ and $\bar{M}$.
By definition of the local model pseudo metric in Equation~\ref{eq:local-model-pseudo-metric}, the maximum possible distance is given by
\begin{equation*}
\max_{M, \bar{M} \in \M^2} \modiv{s}{a}{\gamma V^*_{\bar{M}}}{M}{\bar{M}} = \frac{1 + \gamma}{1 - \gamma}.
\end{equation*}
But this assumes that \emph{any} transition or reward model can define $M$ and $\bar{M}$.
In other words, the maximization is made on the whole set of possible MDPs.
To illustrate why this is too naive, consider a game within the Arcade Learning Environment~\citep{bellemare2013ale}.
We, as humans, have a strong bias concerning similarity between environments.
If the game changes, we still assume groups of pixels will move together on the screen as the result of game actions.
For instance, we generally discard possible new games $\bar{M}$ that ``teleport'' objects across the screen without physical considerations. 
We also discard new games that allow transitions from a given screen to another screen full of static.
These examples illustrate why the knowledge of $\prior$ is very natural (and also why its precise value may be irrelevant).
The same observation can be made for the ``tight'' experiment of Section \ref{sec:experiments}; the set of possible MDPs is restricted by some implicit assumptions that constrain the maximum distance between tasks.
For instance, in these experiments, all transitions move to a neighboring state and never ``teleport'' the agent to the other side of the gridworld.
Without the knowledge of $\prior$, \lrmax{} assumes such environments are possible and therefore transfer values very cautiously (with the ultimate goal not to under estimate the optimal \Qfun{}, in order to avoid negative transfer).
Overall, the experiments of Section \ref{sec:experiments} confirm this important insight: safe transfer occurs slowly if no a priori is given on the maximum distance between MDPs.
On the contrary, the knowledge of $\prior$ allows a faster and more efficient transfer between environments.

\section{The ``tight'' environment used in experiments of Section~\ref{sec:experiments}}
\label{sec:app:tight}

The tight environment is a $11 \times 11$ grid-world illustrated in Figure~\ref{fig:tight}.
The initial state of the agent is the central cell displayed with an ``S''.
The actions are moving 1 cell in one of the four cardinal directions.
The reward is 0 everywhere, except for executing an action in one of the three teal cells in the upper-right corner.
Each time a task is sampled, a slipping probability of executing another action as the one selected is drawn in $[0, 1]$ and the reward received in each one of the teal cells is picked in $[0.8, 1.0]$.

\begin{figure}[]
	\centering
	\includegraphics[width=0.25\textwidth]{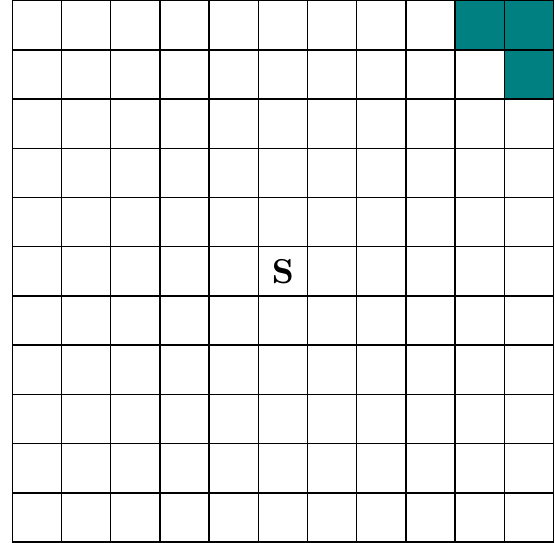}
	\caption{The tight grid-world environment.}
	\label{fig:tight}
\end{figure}

\section{Additional \lrl{} experiments}
\label{sec:app:additional-experiments}

We ran additional experiments on the corridor grid-world environment represented in Figure~\ref{fig:corridor}.
The initial state of the agent is the central cell labeled with the letter ``S''.
The actions are \{left, right\} and the goal is to reach the cell labeled with the letter ``G'' on the extreme right.
A reward $R > 0$ is received when reaching the goal and $0$ otherwise.
At each new task, a new value of $R$ is sampled in $[0.8, 1]$.
The transition function is fixed and deterministic.

\begin{figure}[]
	\centering
	\includegraphics[width=0.3\textwidth]{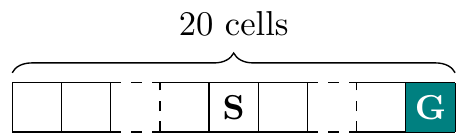}
	\caption{The corridor grid-world environment.}
	\label{fig:corridor}
\end{figure}

The key insight in this experiment is not to lose time exploring the left part of the corridor.
We ran 20 episodes of 11 time steps for each one of the 20 sampled tasks.
Results are displayed in Figure~\ref{fig:corridor-discounted_return_vs_task} and~\ref{fig:corridor-discounted_return_vs_episode}, respectively for the average relative discounted return over episodes and over tasks.
Similarly as in Section~\ref{sec:experiments}, we observe in Figure~\ref{fig:corridor-discounted_return_vs_task} that \lrmax{} benefits from the transfer method as early as the second task.
The \maxqinit{} algorithm benefits from the transfer from task number 12.
Prior knowledge $\prior$ decreases the sample complexity of \lrmax{} as reported earlier and the combination of \lrmax{} with \maxqinit{} outperforms all other methods by providing a tighter \ub{} on the optimal Q-value function.
This decrease of sample complexity is also observed in the episode-wise display of Figure~\ref{fig:corridor-discounted_return_vs_episode} where the convergence happens more quickly on average for \lrmax{} and even more for \maxqinit{}.
This figure allows to see the three learning stages of \lrmax{} reported in Section~\ref{sec:experiments}.
\begin{figure*}[t]
	\centering
	\begin{subfigure}{.49\textwidth}
		\centering
		\includegraphics[
		width=\textwidth%,clip,trim={0 0 20 10}
		]{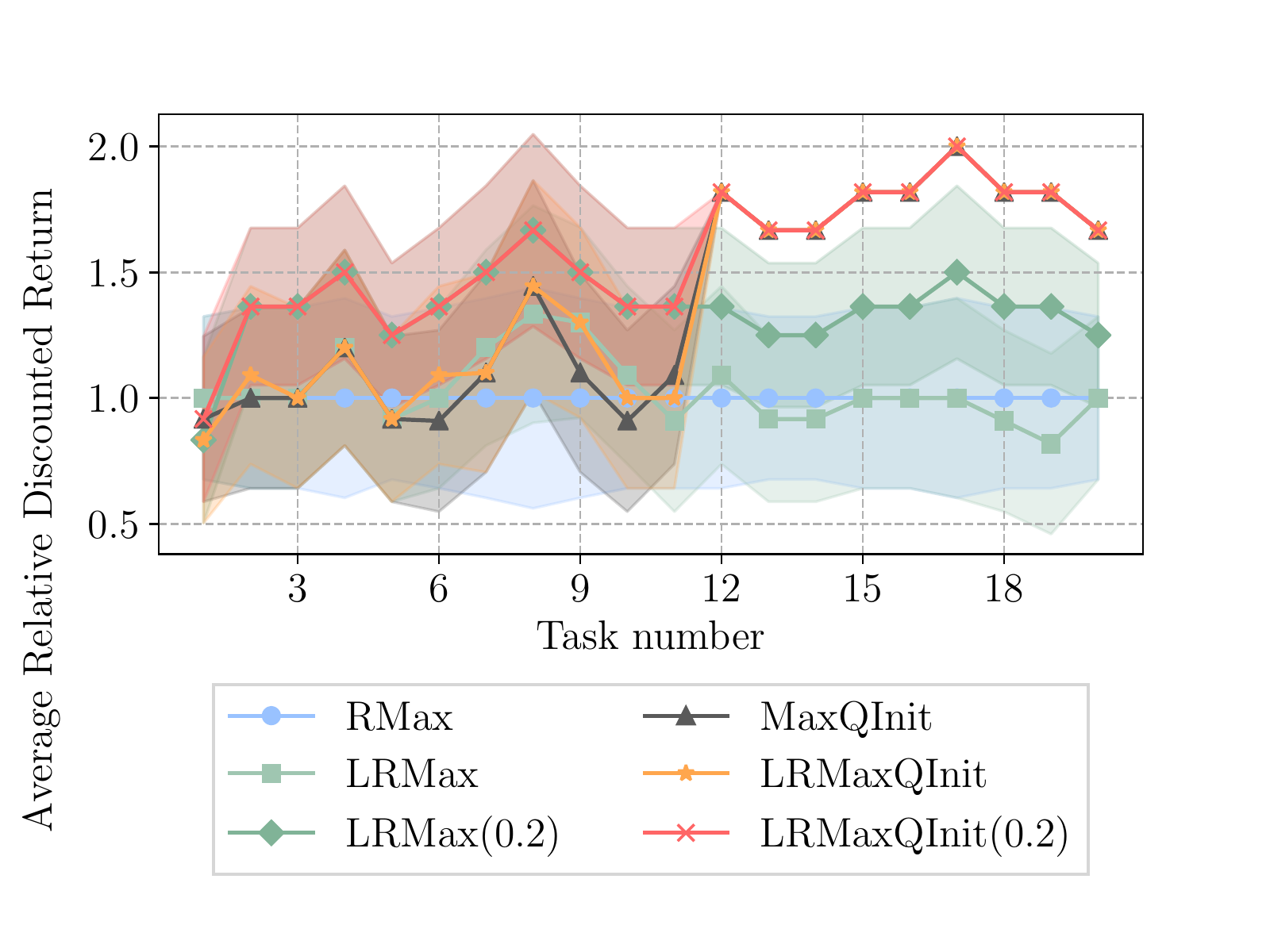}
		\caption{Average discounted return vs. tasks}
		\label{fig:corridor-discounted_return_vs_task}
	\end{subfigure}
	\hfill
	\begin{subfigure}{.49\textwidth}
		\centering
		\includegraphics[
		width=\textwidth%,clip,trim={0 0 40 40}
		]{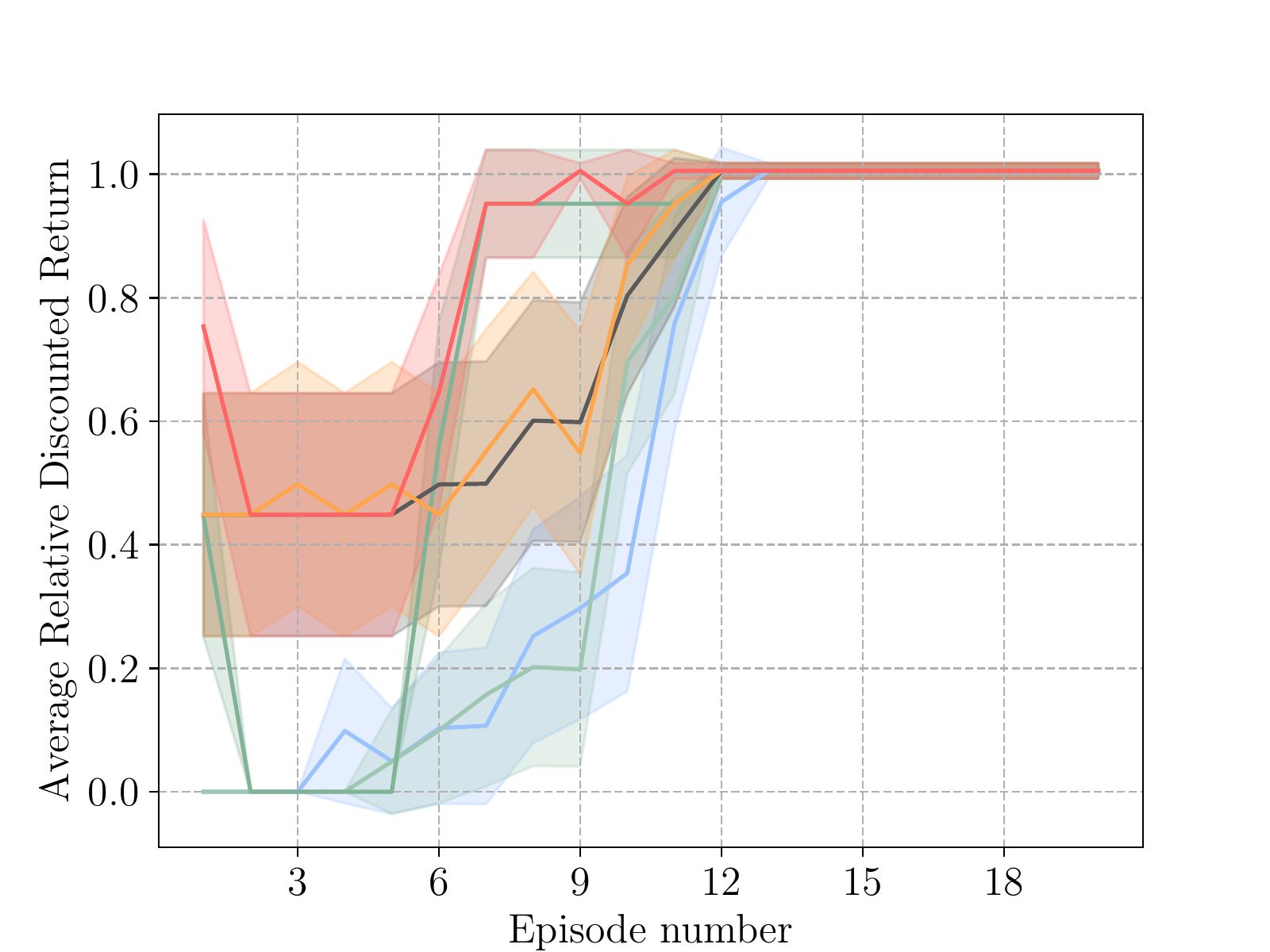}
		\caption{Average discounted return vs. episodes}
		\label{fig:corridor-discounted_return_vs_episode}
	\end{subfigure}
	\caption{Results of the corridor \lrl{} experiment with 95\% confidence interval.}
	\label{fig:additional-experimental-results}
\end{figure*}

% TODO (maybe) remove from here if we do not include the maze experiment

We also ran \lrl{} experiments in the maze grid-world of Figure~\ref{fig:maze}.
The tasks consists in reaching the goal cell labeled with a ``G'' while the initial state of the agent is the central cell, labeled with an ``S''.
Two walls configurations are possible, yielding two different tasks with probability $\frac{1}{2}$ of being sampled in the \lrl{} setting.
The first task corresponds to the case where orange walls are actually walls and green cells are normal white cells where the agent can go.
The second task is the converse, where green walls are walls and orange cells are normal white cells.
We run 100 episodes of length 15 time steps and sample a total of 30 different tasks.
Results can be found in Figure~\ref{fig:maze-results}.
In this experiment, we observe the increase of performance of \lrmax{} as the value of $\prior$ decreases.
The three stages behavior of \lrmax{} reported in Section~\ref{sec:experiments} does not appear in this case.
We tested the performance of using the online estimation of the local model distances of Proposition~\ref{proposition:max-distance-estimate} in the algorithm referred by \lrmax{} in Figure~\ref{fig:maze-results}.
Once enough tasks have been sampled, the estimate on the model local distance is used with high confidence on its value and refines the \ub{} computed analytically in Equation~\ref{eq:local-distance-ub}.
Importantly, this instance of \lrmax{} achieved the best result in this particular environment, demonstrating the usefulness of this result.
This method being similar to the \maxqinit{} estimation of maximum Q-values, we unsurprisingly observe that both algorithms feature a similar performance in the maze environment.

\begin{figure}[]
	\centering
	\includegraphics[width=0.3\textwidth]{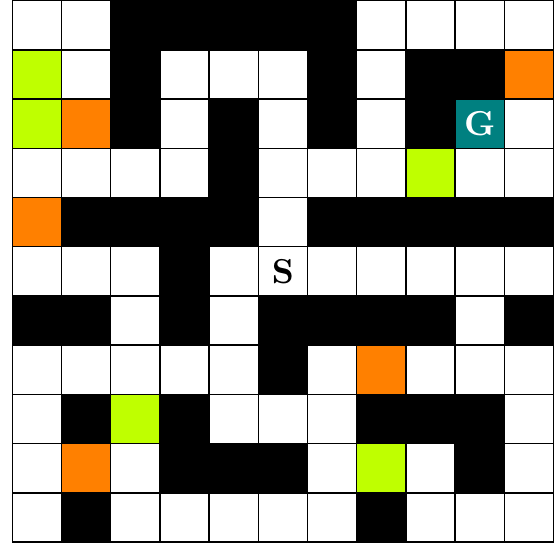}
	\caption{
		The maze grid-world environment.
		The walls correspond to the black cells and either the green ones or the orange ones.
	}
	\label{fig:maze}
\end{figure}

\begin{figure}[]
	\centering
	\includegraphics[width=0.45\textwidth]{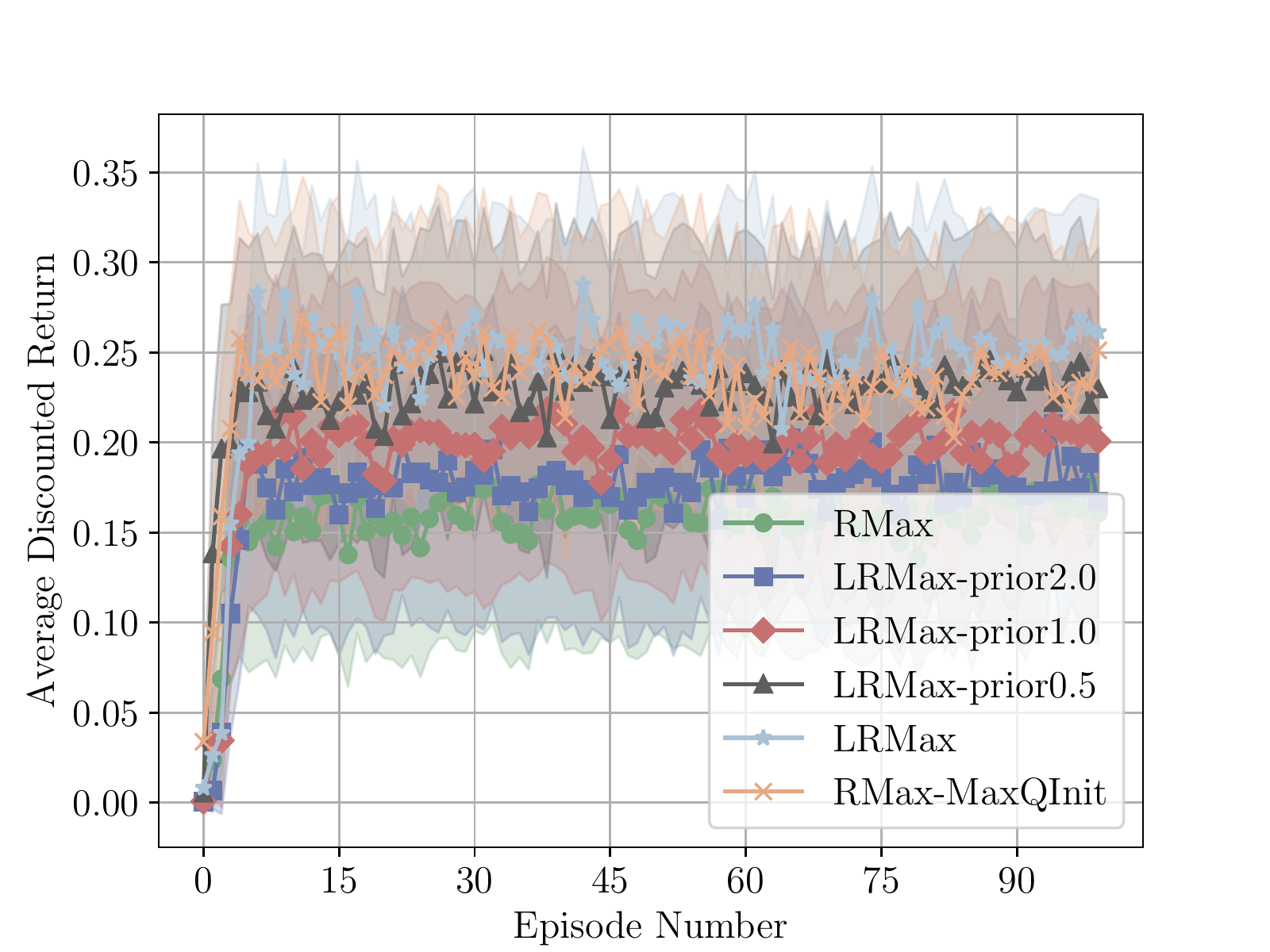}
	\caption{Averaged discounted return over tasks for the maze grid-world \lrl{} experiment.}
	\label{fig:maze-results}
\end{figure}

\section{Prior $\prior$ use experiment}
\label{sec:app:prior-use-experiment}

Consider two MDPs $M, \bar{M} \in \M$.
Each time a state-action pair $\tuple{s}{a} \in \SA$ is updated, we compute the local distance \ub{} $\modivhat{s}{a}{M}{\bar{M}}$ (Equation~\ref{eq:local-distance-ub}) for all $\tuple{s}{a} \in \S \times \A$.
In this computation, one can leverage the knowledge of $\prior$ to select $\min \SET{ \modivhat{s}{a}{M}{\bar{M}}, \prior }$.
We show that \lrmax{} relies less and less on $\prior$ as knowledge on the current task increases.
For this experiment, we used the two grid-worlds environments displayed in Figures~\ref{fig:heat-map-maze-gridworld-1} and~\ref{fig:heat-map-maze-gridworld-2}.
\begin{figure}
	\centering
	\begin{subfigure}{.49\textwidth}
		\centering
		\includegraphics[
		width=0.3\textwidth
		]{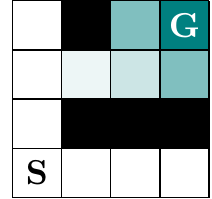}
		\caption{4 times 4 heat-map grid-world. Slipping probability is 10\%.}
		\label{fig:heat-map-maze-gridworld-1}
	\end{subfigure}
	\hfill
	\begin{subfigure}{.49\textwidth}
		\centering
		\includegraphics[
		width=0.3\textwidth
		]{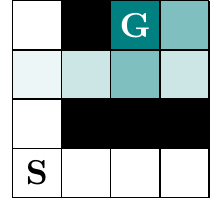}
		\caption{4 times 4 heat-map grid-world. Slipping probability is 5\%.}
		\label{fig:heat-map-maze-gridworld-2}
	\end{subfigure}
	\caption{The two grid-worlds of the prior use experiment.}
\end{figure}

The rewards collected with any actions performed in the teal cells of both tasks are defined as:
\begin{align*}
& R^s_a = \exp \left( - \frac{(s_x - g_x)^2 + (s_y - g_y)^2}{2 \sigma^2} \right), \\
& \forall s = (s_x, s_y) \in \S, a \in \A,
\end{align*}
where $(s_x, s_y)$ are the coordinates of the current state, $(g_x, g_y)$ the coordinate of the goal cell labelled with a G and $\sigma$ is a span parameter equal to $1$ in the first environment and $1.5$ in the second environment.
The agent starts at the cell labelled with the S letter. Black cells represent unreachable cells (walls).
We run \lrmax{} twice on the two different maze grid-worlds and record for each model update the proportion of times $\prior$ is smaller than $\modivhat{s}{a}{M}{\bar{M}}$ in Figure~\ref{fig:prior_use} via the \% use of $\prior$.
\begin{figure}[]
	\centering
	\includegraphics[width=0.45\textwidth]{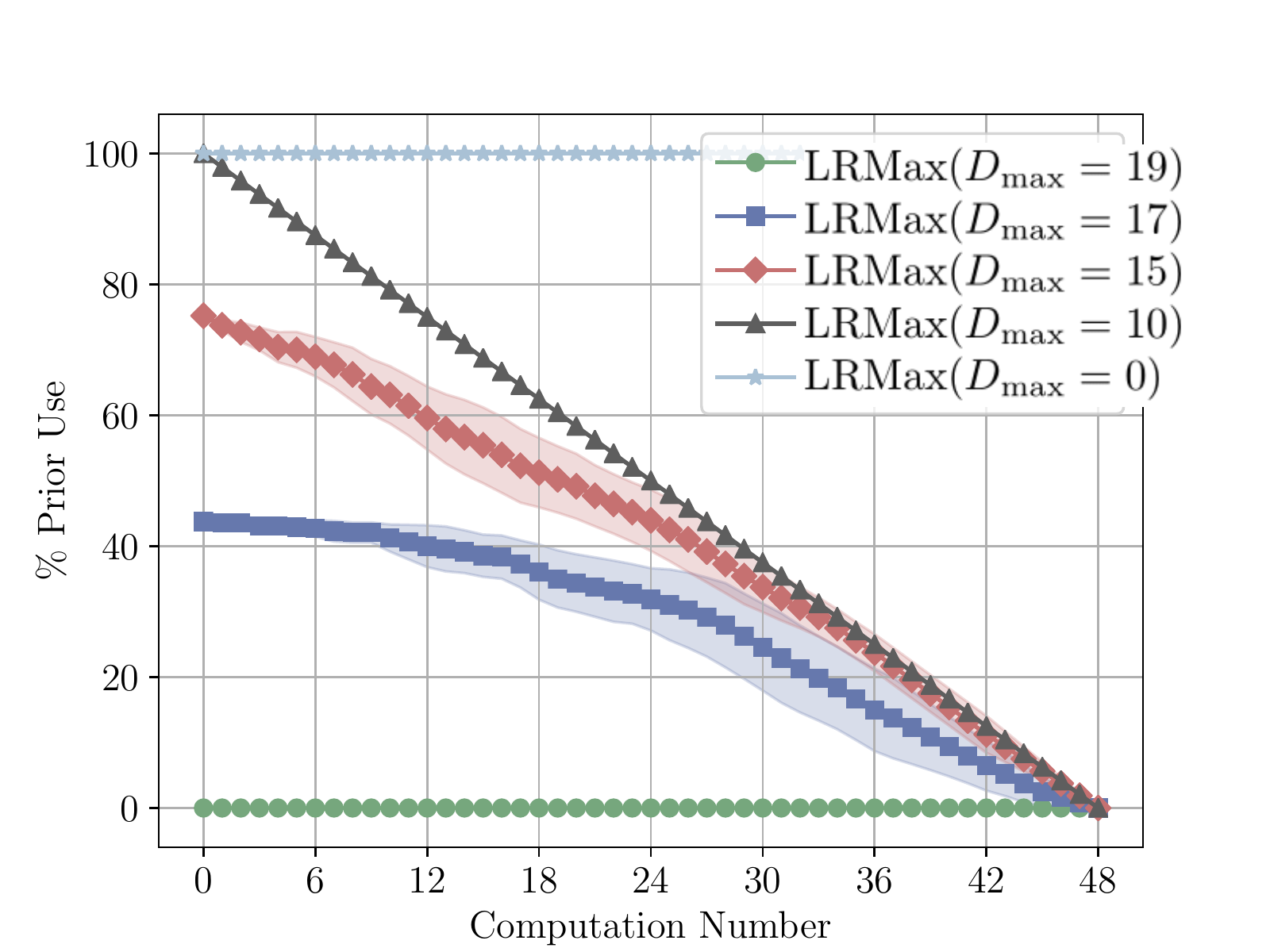}
	\caption{Proportion of times where $\prior \leq \modivhat{s}{a}{M}{\bar{M}}$, \ie{}, use of the prior, vs computation of the Lipschitz bound. Each curve is displayed with 95\% confidence intervals.}
	\label{fig:prior_use}
\end{figure}

With maximum value $\prior = 19$, $\modivhat{s}{a}{M}{\bar{M}}$ is systematically lesser than $\prior$, resulting in 0\% use.
Conversely, with minimum value $\prior = 0$, the use expectedly increases to 100\%.
The in-between value of $\prior = 10$ displays a linear decay of the use.
This suggests that, at each update, $\modivhat{s}{a}{M}{\bar{M}} \leq \prior$ is only true for one more unique $s, a$ pair, resulting in a constant decay of the use.
With fewer prior ($\prior = 15$ or $17$), updating one single $s, a$ pair allows $\modivhat{s}{a}{M}{\bar{M}}$ to drop under $\prior$ for more than one pair, resulting in less use of the prior knowledge.
The conclusion of this experiment if that $\prior$ is only useful at the beginning of the exploration, while \lrmax{} relies more on its own bound $\modivhat{s}{a}{M}{\bar{M}}$ when partial knowledge of the task has been acquired.

\section{Discussion on \rmax{} precision parameters $\epsilon$, $\delta$, $n_{known}$}
\label{sec:app:nknown}

We used $n_{known} = 10$, $\delta=0.05$ and $\epsilon = 0.01$.
Theoretically, $n_{known}$ should be a lot larger ($\approx 10^5$) in order to reach an accuracy $\epsilon = 0.01$ according to \citet{strehl2009reinforcement}.
However, it is common practice to assume such small values of $n_{known}$ are sufficient to reach an acceptable model accuracy $\epsilon$.
% we fixed the value of $\epsilon$ for the sake of the feasibility of the experiment.
Interestingly, empirical validation did not confirm this assumption for any \rmax{}-based algorithm.
We keep these values nonetheless for the sake of comparability between algorithms and consistency with the literature.
Despite such absence of accuracy guarantees, \rmax{}-based algorithms still perform surprisingly well and are robust to model estimation uncertainties.
%This raises the interesting question of \rmax{}'s robustness to badly estimated 
%We keep these values nonetheless for the sake of comparability between algorithms and consistency with the literature.
%We argue that it only plays in disfavor of the \lrmax{} algorithm and does not interfere with the empirical results.

\section{Information about the Machine Learning reproducibility checklist}
\label{sec:app:reproducibility-checklist}

\newcommand{\specialcell}[2][c]{\begin{tabular}[#1]{@{}c@{}}#2\end{tabular}}

For the experiments run in Section~\ref{sec:experiments}, the computing infrastructure used was a laptop using a single 64-bit CPU (model: Intel(R) Core(TM) i7-4810MQ CPU @ 2.80GHz).
The collected samples sizes and number of evaluation runs for each experiment is summarized in Table~\ref{table:summary}.
\begin{table*}[t]
	\centering
	\begin{tabular}{lccccc}
		\toprule
		
		Task &
		\specialcell[3]{Number of\\experiment\\repetitions} &
		\specialcell{Number of\\sampled tasks} &
		\specialcell{Number of\\episodes} &
		\specialcell[3]{Maximum\\length\\of episodes} &
		\specialcell[4]{Total number of\\ collected transition\\samples $(s, a, r, s')$}\\
		
		\midrule
		
		\specialcell{``Tight'' task\\of Figures~\ref{fig:discounted_return_vs_task}~\ref{fig:discounted_return_vs_episode} \\ and~\ref{fig:custom_fig}} & 10 & 15 & 2000 & 10 & 3,000,000 \\
		
		\midrule
		
		\specialcell{``Tight'' task\\of Figure~\ref{fig:bounds_comparison}} & 100 & 2 & 2000 & 10 & 4,000,000 \\
		
		\midrule
		
		\specialcell{Corridor task\\Section~\ref{sec:app:additional-experiments}} & 1 & 20 & 20 & 11 & 4400 \\
		
		\midrule
		
		\specialcell{Maze task\\Section~\ref{sec:app:additional-experiments}} & 1 & 30 & 100 & 15 & 45000 \\
		
		\midrule
		
		\specialcell{Heat-map\\Section~\ref{sec:app:prior-use-experiment}} & 100 & 2 & 100 & 30 & 600,000 \\
		
		\bottomrule
	\end{tabular}
	\captionof{table}{Summary of the number of experiment repetition, number of sampled tasks, number of episodes, maximum length of episodes and \ub{}s on the number of collected samples.}
	\label{table:summary}
\end{table*}

The displayed confidence intervals for any curve presented in the paper is the 95\% confidence interval~\citep{neyman1937x} on the displayed mean.
No data were excluded neither pre-computed.
Hyper-parameters were determined to our appreciation, they may be sub-optimal but we found the results convincing enough to display interesting behaviors.

%\bibliography{llrlbib}
%\bibliographystyle{icml2020}

\fi

\end{document}